\documentclass[11pt]{article}
\usepackage[utf8]{inputenc}
\usepackage[T1]{fontenc}

\usepackage{amsmath,amssymb,amsthm,mathtools}

\usepackage{fullpage}
\usepackage{hyperref}
\usepackage{bm}
\usepackage{comment}
\usepackage{xcolor}
\usepackage{enumitem}
\usepackage{cite}
\usepackage{mathrsfs}

\newcommand{\C}{\mathbb{C}}
\newcommand{\D}{\mathbb{D}}
\newcommand{\R}{\mathbb{R}}
\newcommand{\T}{\mathbb{T}}

\renewcommand{\Re}{\operatorname{Re}} 
\newcommand{\Hcal}{\mathcal{H}} 

\newcommand{\inner}[1]{\langle #1\rangle}
\newcommand{\norm}[1]{\lVert#1\rVert}

\newtheorem{theorem}{Theorem}[section]
\newtheorem{lemma}[theorem]{Lemma}
\newtheorem{proposition}[theorem]{Proposition}
\newtheorem{corollary}[theorem]{Corollary}
\theoremstyle{definition}
\newtheorem{definition}[theorem]{Definition}
\newtheorem{remark}[theorem]{Remark}

\newtheorem{assumption}[theorem]{Assumption}

\allowdisplaybreaks

\setlist{itemsep=4pt,parsep=4pt}

\begin{document}

\title{Infinite-Dimensional Operator/Block Kaczmarz Algorithms :\\ Regret Bounds and \(\lambda\)-Effectiveness}
\author{Halyun Jeong\thanks{Department of Mathematics \& Statistics, University at Albany, State University of New York, Albany, NY, USA.}
\and Palle E.\,T.\ Jorgensen\thanks{Department of Mathematics, The University of Iowa, Iowa City, IA, USA.}
\and Hyun\,-Kyoung Kwon\footnotemark[1]
\and Myung\,-Sin Song\thanks{Department of Mathematics \& Statistics, Southern Illinois University Edwardsville, Edwardsville, IL, USA.}}

\maketitle

\begin{abstract}
We present a variety of projection-based linear regression algorithms with a focus on modern machine-learning models and their algorithmic performance. We study the role of the relaxation parameter in generalized Kaczmarz algorithms and establish a priori regret bounds with explicit $\lambda$-dependence to quantify how much an algorithm's performance deviates from its optimal performance. A detailed analysis of relaxation parameter is also provided. Applications include: explicit regret bounds for the framework of Kaczmarz algorithm models, non-orthogonal Fourier expansions, and the use of regret estimates in modern machine learning models, including for noisy data, i.e., regret bounds for the noisy Kaczmarz algorithms. Motivated by machine-learning practice, our wider framework treats bounded operators (on infinite-dimensional Hilbert spaces), with updates realized as (block) Kaczmarz algorithms, leading to new and versatile results.
\end{abstract}


\section{Introduction}

A significant challenge in models for Regret Bounds and Reinforcement Learning (RL) is closely tied with the general issue of trade-offs between the two parts, exploration and exploitation \cite{russo2018tutorial, shalev2012online}. In rough outline, exploration encourages the player to try new actions, helping future planning, perhaps thereby sacrificing immediate rewards. By contrast, exploitation focuses on maximizing current rewards by utilizing information of known states and actions, but which may prevent the player from learning more information about the game which in turn could help to increase future rewards. To optimize cumulative rewards, the player must balance the two.

This general issue is the cornerstone of broader field of online learning, where algorithms must make decisions sequentially without full knowledge of future events. In this paradigm, performance is often measured by regret, which serves to quantify the difference between an algorithm's cumulative loss and that of an optimal strategy with access to all information in hindsight. A low-regret algorithm, therefore, is one that learns efficiently from data as it arrives, quickly adapting to minimize its cumulative error. The Kaczmarz algorithm, by processing one data point (or a small batch) at each step, operates as a type of online methods; it can be viewed as a form of stochastic gradient descent (SGD) with a batch size of one \cite{needell2014stochastic}, a fundamental tool in modern machine learning.

While the Kaczmarz algorithm does not actively ``explore'' in the RL sense, its performance, when viewed through the lens of regret, shares this spirit of sequential decision-making. The algorithm's performance at each step is based only on past information, and the cumulative error, or regret, measures the cost of this online setting. The purpose of our present paper is to offer new a priori estimates for these regret bounds for the infinite-dimensional Kaczmarz algorithm and its generalizations. This is in a context of new algorithms making use of such Hilbert space tools as non-commutative operator theory, frames, and generalized harmonic analysis.

\smallskip
\noindent
In the present paper, we derive \emph{infinite‑dimensional} regret estimates for a broad class of Kaczmarz‑type algorithms, including block versions and their relaxed counterparts with step size \(\lambda\in(0,2)\).  
Our point of departure is the classical Kaczmarz method, which iteratively projects onto hyperplanes defined by successive data points and has long been used in tomographic reconstruction and stochastic approximation. 
Recent work has rejuvenated the Kaczmarz perspective in machine learning contexts--both for deterministic updates \cite{herr2019positive,jorgensen2020kaczmarz} and for stochastic or noisy variants \cite{strohmer2009randomized,needell2014stochastic, MR3049487,MR4835057}. While these studies focus mainly on finite-dimensional settings, modern ML models (e.g.\ kernel expansions, Fourier frames, non‑orthogonal dictionaries) are most naturally formulated on infinite‑dimensional Hilbert spaces, where the subtleties of operator theory, non‑commutative analysis, and frame theory come into play.

\smallskip
\noindent

More concretely, we address the following questions:

\begin{itemize}
\item \textbf{Regret bounds and $\lambda$-effectiveness.}  
      We prove an \(O(1/k)\) average-regret bound that is dimension free and holds for the generalized Kaczmarz algorithm with any relaxation parameter \(\lambda \in (0,2)\). We also show that when the algorithm achieves meaningful recovery (equivalently, when the system is \(\lambda\)-effective), the bound is sharp when the task/measurement operators are partial isometries.

\item \textbf{Robustness to noise.}  
      For i.i.d.\ additive noise we obtain a mixed regret bound that separates the estimation term from the noise that is proportional to the strength of the noise. This in turn allows us to provide a recommended relaxation parameter choice in terms of the noise level and the number of iterations.

\item \textbf{Effectiveness of Fourier exponential functions.}  
    In many machine learning and signal‑processing applications, redundancy is critical for making predictions or conclusions robust to noise. One canonical example of a redundant system is a frame, whose elements need not be orthogonal. However, other constructions arise by applying the classical Kaczmarz method to systems such as the Fourier exponential functions \(\{e^{2\pi i n x}\}_{n\ge0}\). We extend these results to the generalized Kaczmarz method with relaxation parameter \(\lambda\).

    More precisely, for the generalized Kaczmarz method with relaxation parameter \(\lambda\), we show that the exponential system \(\{e^{2\pi i n x}\}_{n\ge0}\) is \(\lambda\)-effective if and only if the corresponding Borel probability measure is singular unless $\lambda = 1$ (the regular Kaczmarz case). The resulting Fourier‑type expansions, discussed in Sections 8 and 9, provide constructive representations even when orthogonality fails, shedding new light on non‑orthogonal harmonic analysis. 

\end{itemize}

These results supply a potentially unified operator‑level framework for understanding projection‑based learning algorithms, complementing and extending the finite‑dimensional theory while opening the door to applications in signal processing, kernel methods, and infinite‑dimensional RL.

In our analysis, we are inspired in part by recent papers by E. Weber et al. \cite{hegde2021kaczmarz,herr2019positive,herr2020harmonic,herr2017fourier}. In these papers, the authors have initiated the use of the extended Kaczmarz algorithms for dealing with problems in optimization, and in building generalized Fourier frame expansions. We also mention related work \cite{jorgensen2020kaczmarz,jorgensen2023mathematics} on multilevel systems—data, and scaling, by two of the present authors.
We have tried to keep the presentation as elementary and self-contained as possible, since our paper studies regret bounds for iterative methods alongside Hardy-space techniques.

\subsection{Organization}
Our paper is organized as follows: In Section~2, we present our general a priori average regret bound (Theorem~1). Here, for the noiseless case, we give our $O(1/k)$-a priori bound for the norm-squared terms in the average regret, i.e., as computed from the recursive Kaczmarz based algorithm. 

Section 3 begins with the setup for our main iteration steps. In our individual results, we present corresponding a priori bounds for each of the following several steps: (1) data with i.i.d.\ additive noise (2) construction of the noisy Kaczmarz updates (3) error decomposition into estimation and noise parts (4) identification of noiseless vs.\ noisy terms (5) the role played by this splitting in the bound (6) use of the independence of noise to eliminate cross–terms and (7) the resulting noise floor under a uniform pseudoinverse assumption.

In Section~4 we study the generalized (relaxed) Kaczmarz with step size $\lambda\in(0,2)$. We derive a key operator identity leading to a scaled Parseval relation and a sharp cumulative–regret identity with constant $1/[\lambda(2-\lambda)]$, and obtain a noisy average–regret bound with explicit $\lambda$–dependence (including a noise–aware discussion of the choice of~$\lambda$).

In the remaining sections we develop the inner–function framework for $\lambda$–effectiveness and apply it to exponential systems. Section~5 introduces the auxiliary sequence $h_n^{(\lambda)}$ for rank–one projections and its recursion; Section~6 identifies and equates the recursive and combinatorial coefficients $\alpha_n^{(\lambda)}$ governing $h_n^{(\lambda)}$; Section~7 proves that $\{e^{2\pi i n x}\}_{n\ge0}$ is $\lambda$–effective for every $\lambda\in(0,2)$ whenever the underlying measure is singular, with a separate discussion of the special case $\lambda=1$; and Section~8 gives generalized Kaczmarz–Fourier expansions and illustrates some of their consequences. Some technical but elementary auxiliary proofs appear in the Appendix.

\section{The Average Regret Bound of the Kaczmarz Algorithm}

In earlier work (see for e.g., \cite{berner2024operator,jorgensen2023mathematics}) a dynamic approach to a generalized family of Kaczmarz algorithms was introduced via sequences of specified projections, creating the successive updates. Instead, our present recursive models here provide a general framework that takes bounded-operator systems as the starting point (in infinite-dimensional Hilbert space), leading to new and more versatile operator (block) Kaczmarz algorithms.

Let $\mathcal{H}$ and $\mathcal{K}$ be Hilbert spaces. In our framework, we consider a parameter vector \(w^*\) in $\mathcal{H}$ and a \emph{measurement operator or task} \(X_t: \mathcal{H} \rightarrow \mathcal{K} \), which maps \(w^*\) to the observation \(y_t \in \mathcal{K} \) via
\[
y_t = X_t\,w^*,
\] where \(t=1,2,\dots,k\).

Here, $\tau$ is
a \emph{task–selection rule}, i.e., any map $\tau:\mathbb{N}\to I$ for some index set $I$; it specifies which task is used at time $t$. Examples include the cyclic one $\tau_{\mathrm{cyc}}^{(m)}(t)=1+((t-1)\bmod m)$, when $I=\{1,\dots,m\}$, or simply a naturally ordered sequence when set $\tau(t)=t$ and $I = \mathbb{N}$. 
For the remainder of our analysis, once the task sequence is fixed, we simplify notation by letting the time index $t$ also index the operators: define $X_t:=X_{\tau(t)}$, $y_t:=y_{\tau(t)}$, and $P_t:=P_{\tau(t)}$. All bounds in Sections 2–5 hold for any fixed, arbitrary sequence $(P_t)_{t\ge 1}$; the constants in Theorems are independent of $t$. 

\begin{assumption}
\label{asmp:task_bound}
For each $t$, the measurement operator $X_t$ satisfies $\|X_t\| \le 1$. \footnote{This assumption is a standard assumption typically used in theoretical analysis of machine learning algorithms such as \cite{evron2022catastrophic} to make the presentation simple. It is also without loss of generality: one may rescale all operators and observations by a common factor if necessary so that the norms are at most one.
}
\end{assumption}

\begin{assumption}
\label{asmp:pinv}
For each $t$, the Moore--Penrose pseudoinverse $X_t^\dagger$ \footnote{ The Moore-Penrose pseudoinverse $X^\dagger$ is the unique operator satisfying the Moore-Penrose conditions: \[
XX^\dagger X = X,\quad X^\dagger X X^\dagger = X^\dagger,\quad
(XX^\dagger)^* = XX^\dagger,\quad (X^\dagger X)^* = X^\dagger X.
\]} \footnote{See \cite{MR4506052, MR3921904, MR1892841} for more recent application/analysis of the Moore-Penrose pseudoinverse.} exists as a bounded operator and
there exists a constant \(C<\infty\) such that \(\|X_t^\dagger\|^2 \le C\) for all \(t\).
Consequently,
\[
P_t:=X_t^\dagger X_t\qquad
\text{is the orthogonal projector onto the row space of }X_t. \footnote{Equivalently,
\(X^\dagger X = P_{(\ker X)^\perp} = P_{\overline{\operatorname{ran}(X^*)}}\)}.
\]
\end{assumption}

This assumption holds, in particular, if each $X_t$ is of finite rank with a uniform lower bound
on the smallest nonzero singular value, i.e.\ $\inf_t \sigma_{\min}(X_t)>0$, in which case
$\|X_t^\dagger\|=1/\sigma_{\min}(X_t)$ and one may take
$C=(\inf_t \sigma_{\min}(X_t))^{-2}$.

The Kaczmarz algorithm (in its block form) \footnote{See e.g., \cite{needell2014paved} for the finite dimensional version.} updates the model via the recursion
\[
w_t \;=\; w_{t-1} \;+\; X_t^\dagger \Bigl(y_t - X_t\,w_{t-1}\Bigr).
\]
When \(X_t\) is a single vector \(x_t\), this iteration reduces to the standard (row-by-row) Kaczmarz update:
\[
w_t \;=\; w_{t-1} \;+\; \frac{x_t}{\|x_t\|_2^2}\Bigl(y_t - \langle x_t,\,w_{t-1}\rangle\Bigr).
\]

At each iteration $t$, the algorithm, having produced the parameter vector $w_{t-1}$ from the first $t-1$ tasks, incurs a loss on the current task. The per-step regret at iteration $t$ is the squared prediction error, which is the standard formulation of regret at time $t$ for linear regression \cite{evron2022catastrophic, shalev2012online}:
\[
\text{Regret}_t := \norm{X_t w_{t-1} - y_t }^2.
\]
Crucially, the prediction for task $X_t$ is made using $w_{t-1}$, the model parameter \emph{before} it has been updated with the information from $(X_t, y_t)$.

The \textbf{average regret} after $k$ iterations is the average of these per-step costs is given by
\[
\text{AverageRegret}_k := \frac{1}{k} \sum_{t=1}^k \norm{X_t w_{t-1} - y_t}^2.
\]
Note that this definition of regret, which measures the cost of sequential prediction, is distinct from the notion of ``catastrophic forgetting,'' which typically measures performance degradation on past tasks using the final iterate $w_k$ \cite{evron2022catastrophic}, where the authors also focus only on the finite dimensional setting.

The following generalizes the definition of effectiveness in \cite{jorgensen2020kaczmarz}:
\begin{definition}[Effectiveness and $\lambda$--effectiveness]
\label{def:lambda effectiveness}
Let $\mathcal{H}$ be a Hilbert space and let $(P_n)_{n\ge1}$ be orthogonal projections on $\mathcal{H}$.
Define
\[
T_n:=(I-P_n)\cdots(I-P_1),\qquad
\widetilde T_n:=(I-\lambda P_n)\cdots(I-\lambda P_1),\ \ \lambda\in(0,2).
\]
We say the system $(P_n)_{n\ge1}$ or $(X_n)_{n\ge1}$ if $P_n = X_n^\dagger X_n$, is \emph{effective} if the corresponding $T_n$ converges to $0$ as $n \rightarrow \infty$ in the strong operator topology (denoted by $T_n \xrightarrow{\text{SOT}} 0$), and \emph{$\lambda$--effective} if $\widetilde T_n \xrightarrow{\text{SOT}} 0$.
\end{definition}

\begin{theorem}
\label{thm:kaczmarz}
Let $\{w_t\}_{t=0}^k$ be the sequence of updates by the (block) Kaczmarz algorithm with the initial condition $w_0 = 0$. Suppose there is a parameter vector $w^*$ such that $y_t = X_tw^*$ for all $t = 1, 2, \dots, k$. Suppose that Assumption~\ref{asmp:task_bound} holds with $\|X_t\| \le 1$ and Assumption~\ref{asmp:pinv} holds with $\sup_{t} \|X_t^\dagger\|^2 = C < \infty$.

Then the average regret of the Kaczmarz algorithm satisfies
\[
\frac{1}{k} \sum_{t=1}^k \bigl\|X_t w_{t-1} - y_t\bigr\|^2
\;\le\; \frac{1}{k} \|w^*\|^2,
\]
which shows it is bounded by $O(1/k)$.

If we assume further that the system $\{ P_t \}_{t \ge 0}$ is effective and the associated measurement operator $X_t$ is a partial isometry (i.e, $X_t^{*}X_t=P_t$), then the above bound is sharp in the following sense:

\[
\sum_{t=1}^\infty \bigl\|X_t w_{t-1} - y_t\bigr\|^2
\;=\; \|w^*\|^2.
\]
\end{theorem}

\begin{remark}[Rank-one partial isometries]
\label{rmk:rank-one-partial-isometry}
Let $X:\mathcal H\to\mathcal K$ be of the form $Xh=\langle h,v\rangle\,u$ with nonzero $u\in\mathcal K$, $v\in\mathcal H$. Then
\[
X^*y=\langle y,u\rangle v,\qquad
X^\dagger=\frac{1}{\|u\|^2\|v\|^2}\,X^*,
\]
and in general
\[
X^*X=(\|u\|^2\|v\|^2)\,P_{\mathrm{span}\{v\}},\qquad
XX^*=(\|u\|^2\|v\|^2)\,P_{\mathrm{span}\{u\}}.
\]
Hence the following are equivalent:
\[
X\ \text{is a partial isometry}
\ \Longleftrightarrow\
\|u\|\,\|v\|=1
\ \Longleftrightarrow\
X^\dagger=X^*
\ \Longleftrightarrow\
X^*X=P_{\mathrm{span}\{v\}} \ (\text{and } XX^*=P_{\mathrm{span}\{u\}}).
\]
In particular, under these equivalent conditions we may write $P:=X^*X=P_{\mathrm{span}\{v\}}$, which is exactly the case used later in our theorems.
\end{remark}

\begin{proof} [Proof of Theorem~\ref{thm:kaczmarz}]
Define the error vectors
\[
\epsilon_t \;=\; w^* - w_t,\qquad \text{for } t=0,1,\ldots,k.
\]
Since $w_0 = 0$, we have $\epsilon_0 = w^* - w_0 = w^*$. From the update equation
\[
w_t \;=\; w_{t-1} \;+\; X_t^\dagger \bigl(y_t - X_tw_{t-1}\bigr),
\]
and using $X_tw^* = y_t$, it follows that
\[
w_t - w^*
\;=\; \bigl(I - X_t^\dagger X_t\bigr)\,\bigl(w_{t-1} - w^*\bigr).
\]
Hence, in terms of $\epsilon_t$:
\[
\epsilon_t \;=\; (I - P_t)\, \epsilon_{t-1},
\]
where $P_t = X_t^\dagger\,X_t$ is the orthogonal projector associated with $X_t$. By induction on $t$,
\[
\epsilon_t
\;=\; (I - P_t) \,(I - P_{t-1}) \cdots (I - P_1)\, \epsilon_0
\;=\; (I - P_t) \,(I - P_{t-1}) \cdots (I - P_1)\, w^*.
\]

We now examine the sum of squared residuals:
\[
\sum_{t=1}^k \bigl\|X_t w_{t-1} - y_t\bigr\|^2
\;=\; \sum_{t=1}^k \bigl\|X_t(w_{t-1} - w^*)\bigr\|^2
\;=\; \sum_{t=1}^k \bigl\|X_t\, \epsilon_{t-1}\bigr\|^2.
\]
As we assume that the data matrices satisfy $\|X_t\| \le 1$ by Assumption~\ref{asmp:task_bound} as in \cite{evron2022catastrophic}, for any $u \in \mathcal{H}$, we have
\begin{align}
\label{eq:Data_Projection_Bound}
\|X_t u\|_2^2 = \|X_t X_t^{\dagger} X_t u\|_2^2 \le \|X_t \|_2^2 \| X_t^{\dagger} X_t u\|_2^2 \le \| X_t^{\dagger} X_t u\|_2^2 = \| P_t u\|_2^2. 
\end{align}

Consequently, this implies that
\[
\sum_{t=1}^k \bigl\|X_t\,\epsilon_{t-1}\bigr\|^2
\;\le\; \sum_{t=1}^k \bigl\|P_t\, \epsilon_{t-1}\bigr\|^2.
\]
Now, define
\[
Q_t
\;=\; P_t\, (I-P_{t-1})\, \cdots \,(I-P_1),
\]
as in~\cite{jorgensen2020kaczmarz}.
Recalling $\epsilon_{t-1} = (I-P_{t-1})\cdots(I-P_1) w^*$, we can write
\[
P_t\, \epsilon_{t-1}
\;=\; P_t\, (I-P_{t-1}) \cdots (I-P_1)\, w^* \;= Q_t \, w^*.
\]
Thus,
\[
\sum_{t=1}^k \bigl\|X_t w_{t-1} - y_t\bigr\|^2
\;\le\; \sum_{t=1}^k \bigl\|P_t\, \epsilon_{t-1}\bigr\|^2=\; \sum_{t=1}^k \bigl\|Q_t\, w^*\bigr\|^2.
\]

We next use the following key result from Theorem 3 in~\cite{jorgensen2020kaczmarz}; First, for the sequence $(P_t)_{t=1}^k$, define
\begin{align*}
T_0 &= I,\\
T_k &= (I- P_k)(I- P_{k-1})\cdots (I- P_1), \quad k \ge 1.
\end{align*}
Then Eq. (3.10) in Theorem 3 in~\cite{jorgensen2020kaczmarz} states that 
\[
I - T_k^* T_k = \sum_{t=1}^k Q_t^* Q_t.
\] (We will also provide a self-contained proof for more general form in Lemma~\ref{lem:Generalized_Kaczmarz_Scaled_Parseval} for readers.)

This implies that for all $u \in \mathcal{H}$, we have
\[
\|u\|_2^2 = \|T_k u\|_2^2 + \sum_{t=1}^k \|Q_t u\|_2^2,
\]
so that, for all $k \ge 1$,
\[
\sum_{t=1}^k \|Q_t u\|_2^2 \le \|u\|_2^2.
\]

In particular, for $u = w^*$, this implies
\begin{align}
\label{eq:Q_bound}
\sum_{t=1}^k \bigl\|Q_t\, w^*\bigr\|^2 \;\le\; \|w^*\|^2.
\end{align}

Putting all the previous bounds together,
\[
\frac{1}{k} \sum_{t=1}^k \bigl\|X_t w_{t-1} - y_t\bigr\|^2
\;\le\; \frac{1}{k}\,\|w^*\|^2,
\]
establishing that the average regret is $O\bigl(\tfrac{1}{k}\bigr)$.
This proves the first part of the theorem.

If, in addition, the system \(\{P_t\}_{t\ge 1}\) is effective and each
measurement operator \(X_t\) is a partial isometry (equivalently,
\(X_t^{*}X_t=P_t\)), then for every \(u\in\mathcal H\),
\[
\|X_t u\|^2
= \langle X_t^*X_t u,\,u\rangle
= \langle P_t u,\,u\rangle
= \|P_t u\|^2,
\]
since \(P_t\) is an orthogonal projector.
Consequently, \(\|X_t u\|=\|P_t u\|\) for all \(u \in \mathcal{H}\). Hence, we have

\[
\sum_{t=1}^k \bigl\|X_t w_{t-1} - y_t\bigr\|^2
\;=\; \sum_{t=1}^k \bigl\|P_t\, \epsilon_{t-1}\bigr\|^2=\; \sum_{t=1}^k \bigl\|Q_t\, w^*\bigr\|^2,
\]
and Corollary 3.1 in \cite{jorgensen2020kaczmarz} states that
\[
\sum_{t=1}^\infty \bigl\|Q_t\, w^*\bigr\|^2 = \|w^*\|_2^2.
\] This shows the second part of the theorem.

\end{proof}

\section{Regret Bound for Kaczmarz Algorithm with Noisy Observations}
In this section, we show that the regret bound can be generalized to the noisy case. We assume that the randomness comes only from noise vectors $\eta_1,\eta_2,\dots$, which are independent, mean-zero, and have a common finite second moment $\mathbb E\|\eta_t\|^2=\sigma^2$. Throughout this paper, we write $\mathbb{E}[\cdot]$ for expectation with respect to the noise vectors only. Define $y_t:=X_t w^*+\eta_t$ and
\[
w_t \;=\; w_{t-1}+X_t^\dagger\!\bigl(y_t-X_t w_{t-1}\bigr),\qquad w_0=0.
\]

\begin{theorem}
\label{thm:Regret Bound for the Noisy Kaczmarz Algorithm}
Let $\{(X_t, y_t)\}_{t=1}^k$ be data blocks with
\[
y_t = X_t\,w^* + \eta_t,
\]
where $\{\eta_t\}$ are i.i.d. mean-zero noise vectors satisfying
\[
\mathbb{E}[\|\eta_t\|^2] = \sigma^2 \; \text{for each $t$}.
\]
Suppose that Assumption~\ref{asmp:task_bound} holds with $\|X_t\| \le 1$ and Assumption~\ref{asmp:pinv} holds with $\sup_{t} \|X_t^\dagger\|^2 = C < \infty$.
As before, the Kaczmarz update is given by
\[
w_t = w_{t-1} + X_t^\dagger\Bigl[y_t - X_t\,w_{t-1}\Bigr],\quad w_0 = 0,
\] but with the noisy observations $y_t = X_t\,w^* + \eta_t$. 

Then, the average regret (or squared residual) satisfies
\[
\frac{1}{k}\sum_{t=1}^k \mathbb{E} \|X_tw_{t-1} - y_t\|^2 \le \frac{2\,\|w^*\|^2}{k} + (2C+1) \mathcal{O}(\sigma^2).
\] 
\end{theorem}

\begin{remark}
Note that our noise model assumption is mild; we assume i.i.d.\ mean-zero noise with finite second moment \(\mathbb{E}\|\eta_t\|^2=\sigma^2\) for each $t$. This covers many heavy-tailed distributions beyond Gaussian or sub-Gaussian.
\end{remark}

\begin{remark}
In the noiseless case, for the squared-loss \(l_t(u)=\tfrac12\|X_t u-y_t\|^2\), we have
\[
\sum_{t=1}^k \|X_tw_{t-1} - y_t\|^2
\;=\; 2\sum_{t=1}^k \bigl(l_t(w_{t-1}) - l_t(u)\bigr)
\quad\text{for }u=w^*,
\]
since \(l_t(w^*)=0\).
For the noisy case, for a convex loss \(l_t\),
\[
\sum_{t=1}^k \Re\inner{g_t,\, w_{t-1}-u}
\;\ge\; \sum_{t=1}^k \bigl(l_t(w_{t-1}) - l_t(u)\bigr),
\]
where \(g_t\) is the stochastic gradient. Consequently, recent papers on squared-loss regret use \(\sum_{t=1}^k \inner{g_t, w_{t-1}-u}\) as the regret \cite{zhangunconstrained, mcmahan2014unconstrained}; we follow this convention.

Moreover, let \(g_t := \overline{X_t}^\top \!\bigl(X_tw_{t-1}-y_t\bigr)\), the stochastic gradient of the Kaczmarz algorithm for \(l_t(u)=\tfrac12\|X_t u-y_t\|^2\). Then, for any \(u\),
\[
\inner{g_t,\, w_{t-1}-u}
= \|X_tw_{t-1}-y_t\|^2
\;-\; \inner{\,X_tw_{t-1}-y_t,\, X_t u-y_t\,}.
\]
With \(u=w^*\) and \(y_t=X_t w^*+\eta_t\),
\[
\inner{g_t,\, w_{t-1}-u}
= \|X_tw_{t-1}-y_t\|^2
\;-\; \inner{\,X_tw_{t-1}-X_t w^*-\eta_t,\, X_t u-X_t w^*-\eta_t\\,}.
\]
Since \(\eta_t\) is mean zero and independent of the past,
\[
\mathbb{E}\bigl[\Re\inner{g_t,\, w_{t-1}-w^*}\bigr]
\;=\; \mathbb{E}\,\|X_tw_{t-1}-y_t\|^2 \;-\; \mathbb{E}\|\eta_t\|^2 = \mathbb{E}\,\|X_tw_{t-1}-y_t\|^2 \;-\; \sigma^2,
\]
or equivalently,
\[
\mathbb{E}\bigl[\Re\inner{g_t,\, w_{t-1}-w^*}\bigr]
\;=\; 2\,\mathbb{E}\bigl[l_t(w_{t-1})-l_t(w^*)\bigr].
\]
Furthermore, we note that \(u=w^*\) is the natural optimal choice in noisy setting: it minimizes the expected loss \( \mathbb{E}\,l_t(u)=\tfrac12\bigl(\|X_t(u-w^*)\|^2+\mathbb{E}\|\eta_t\|^2\bigr)\).
\end{remark}

\begin{proof}

The proof of Theorem~\ref{thm:Regret Bound for the Noisy Kaczmarz Algorithm} is organized in several steps.

\medskip
As the first step, we consider data blocks $(X_t, y_t)$, $t=1,\ldots,k$, with
\[
y_t = X_t\,w^* + \eta_t.
\]
Then, the noisy Kaczmarz update can be expressed as
\[
w_t = w_{t-1} + X_t^\dagger\Bigl[X_t\,w^* + \eta_t - X_t\,w_{t-1}\Bigr],\quad w_0 = 0.
\]
As before, we define the error vector as
\[
\epsilon_t = w^* - w_t,
\]
which allows us to express the noisy Kaczmarz update as
\[
\epsilon_t = (I - P_t)\,\epsilon_{t-1} - X_t^\dagger\,\eta_t, \quad \text{where $P_t = X_t^\dagger X_t$. }
\]

In the next step, we decompose the error as follows;

Unrolling the recursion $\epsilon_t = (I - P_t)\,\epsilon_{t-1} - X_t^\dagger\,\eta_t$ followed by the multiplication by $X_t$ leads to
\[
X_t\,\epsilon_{t-1} = A_t - B_t,
\]
where we define
\[
A_t := X_t\Bigl[(I - P_{t-1})\cdots (I - P_1)\Bigr]\,w^*,
\]
and
\begin{align}
\label{eq:noise_term}
B_t := X_t \sum_{i=1}^{t-1} \Bigl[(I - P_{t-1})\cdots (I - P_{i+1})\Bigr]\,X_i^\dagger\,\eta_i.
\end{align}

Using the standard inequality $\|u-v\|^2 \le 2\|u\|^2 + 2\|v\|^2$, we obtain
\[
\|X_t\,\epsilon_{t-1}\|^2 \le 2\,\|A_t\|^2 + 2\,\|B_t\|^2.
\]
Summing both sides of the above bound over $t=1,\ldots,k$ gives
\begin{align}
\label{eq:bound_by_noiseless_noisy_terms}
\sum_{t=1}^k \|X_t\,\epsilon_{t-1}\|^2 \le 2\sum_{t=1}^k \|A_t\|^2 + 2\sum_{t=1}^k \|B_t\|^2.
\end{align}

Then, we express the noiseless term $A_t$ in terms of 
\[
Q_t := P_t\,(I - P_{t-1})\cdots (I - P_1), \; \text{as in the noiseless case.}
\]
Since $\|X_t u\|\le \|P_t u\|$ for any $u$ from \eqref{eq:Data_Projection_Bound} due to Assumption~\ref{asmp:task_bound}, it follows that
\[
\|A_t\|^2 = \Bigl\|X_t (I - P_{t-1})\cdots (I - P_1)\,w^*\Bigr\|^2 \le \|Q_t\,w^*\|^2.
\]
By the bound \eqref{eq:Q_bound} in the noiseless case (or by Eq. (3.11) in Theorem 3 in~\cite{jorgensen2020kaczmarz}),
\[
\sum_{t=1}^k \|Q_t\,w^*\|^2 \le \|w^*\|^2,
\]
so that
\begin{align}
\label{eq:noiseless_bound_theorem3}
\sum_{t=1}^k \|A_t\|^2 \le \|w^*\|^2.
\end{align}

Next for $B_t$ term due to the noise that was defined in \eqref{eq:noise_term}
\[
B_t
\;=\;
X_t
 \sum_{i=1}^{t-1}
 (I - P_{t-1})\cdots (I - P_{i+1})
 \,X_i^\dagger\,\eta_i,
\]
we define the ``partial'' operators:
\[
\widetilde{Q}_t^i
\;:=\;
P_t
\Bigl[
(I - P_{t-1})\cdots (I - P_{i+1})
\Bigr],
\quad
\text{for }1\le i < t\le k.
\]
Then, we have,
\begin{align*}
\|B_t\|_2^2
\;&=\;
\Bigl\|
X_t \sum_{i=1}^{t-1}
\bigl[(I - P_{t-1})\cdots (I - P_{i+1})\bigr]
 X_i^\dagger\,\eta_i
\Bigr\|_2^2\\
\;&\le\; \Bigl\|
P_t \sum_{i=1}^{t-1}
\bigl[(I - P_{t-1})\cdots (I - P_{i+1})\bigr]
X_i^\dagger\,\eta_i
\Bigr\|_2^2\\
\;&=\;
\Bigl\|
\sum_{i=1}^{t-1}
 \widetilde{Q}_t^i\,
\bigl(X_i^\dagger\,\eta_i\bigr)
\Bigr\|_2^2,
\end{align*} where the inequality follows from $\|X_t u\|\le \|P_t u\|$ for any $u$ from \eqref{eq:Data_Projection_Bound}.

Because \(\{\eta_i\}\) are i.i.d.\ with zero mean, for each fixed~\(t\),
\[
\mathbb{E}\Bigl[
\Bigl\|
 \sum_{i=1}^{t-1}
\widetilde{Q}_t^i\,
\bigl(X_i^\dagger\,\eta_i\bigr)
\Bigr\|^2
\Bigr]
\;=\;
\sum_{i=1}^{t-1}
 \mathbb{E}\bigl[
 \|\widetilde{Q}_t^i\,(X_i^\dagger\,\eta_i)\|^2
 \bigr],
\]
since cross‐terms vanish under expectation. Hence, combining the previous bound on $\|B_t\|_2^2$ and taking the expectation yield
\[
\mathbb{E} \|B_t\|_2^2 \le \sum_{i=1}^{t-1}
 \mathbb{E}\bigl[
 \|\widetilde{Q}_t^i\,(X_i^\dagger\,\eta_i)\|^2
 \bigr].
\]

After taking expectation to both sides of the inequality in \eqref{eq:bound_by_noiseless_noisy_terms}, due to the linearity of the expectation, we have

\begin{align}
\label{eq:bound_theorem3}
\sum_{t=1}^k
 \mathbb{E}\|X_t\,\epsilon_{t-1}\|^2
\le
2\sum_{t=1}^k
 \mathbb{E}\|A_t\|^2 + 2 \sum_{t=1}^k
\mathbb{E} \|B_t\|^2 = 2\sum_{t=1}^k
 \|A_t\|^2 + 2 \sum_{t=1}^k
\mathbb{E} \|B_t\|^2,
\end{align}
where the equality is from the fact that $A_t$ does not depend on the randomness of the noise sequence $\{\eta_t\}$.

Using the bound on $\mathbb{E} \|B_t\|^2$ for each $t$ we have derived, the term from the noisy observation, we have
\begin{align*}
\sum\limits_{t=1}^k \mathbb{E} \|B_t\|^2
&\le \sum\limits_{t=1}^k \sum_{i=1}^{t-1}
 \mathbb{E}\bigl[
 \|\widetilde{Q}_t^i\,(X_i^\dagger\,\eta_i)\|^2
 \bigr] = \sum\limits_{i=1}^{k-1} \sum_{t=i+1}^{k}
 \mathbb{E}\bigl[
 \|\widetilde{Q}_t^i\,(X_i^\dagger\,\eta_i)\|^2
 \bigr],
\end{align*}
where the equality is obtained by switching the order of the double sum.
Then, again by the linearity of expectation, we have
\begin{align}
\label{eq:double_sum_switched}
\sum\limits_{i=1}^{k-1} \sum_{t=i+1}^{k}
 \mathbb{E}\bigl[
 \|\widetilde{Q}_t^i\,(X_i^\dagger\,\eta_i)\|^2
 \bigr]
\;=\;
\sum\limits_{i=1}^{k-1} \mathbb{E} \left[ \sum_{t=i+1}^{k}
\bigl[
 \|\widetilde{Q}_t^i\,(X_i^\dagger\,\eta_i)\|^2
 \bigr] \right]. 
\end{align}

Recall that $\widetilde{Q}_t^i$ is defined as
\[
\widetilde{Q}_t^i
\;:=\;
P_t
\Bigl[
(I - P_{t-1})\cdots (I - P_{i+1})
\Bigr],
\quad
\text{for }1\le i < t\le k.
\]
So, by the same argument used to show the bound \eqref{eq:Q_bound} in the noiseless case (or by the same argument for Eq. (3.11) in Theorem 3 in~\cite{jorgensen2020kaczmarz}), we have
\[
\sum_{t=i+1}^{k}
 \|\widetilde{Q}_t^i\,(X_i^\dagger\,\eta_i)\|^2 \le
\bigl\|X_i^\dagger\,\eta_i\bigr\|^2
\;\le\;
\|X_i^\dagger\|_{\mathrm{op}}^2\,
\|\eta_i\|^2.
\]
Using this inequality to  \eqref{eq:double_sum_switched}, we obtain
\[
\sum\limits_{i=1}^{k-1} \mathbb{E} \left[ \sum_{t=i+1}^{k}
\bigl[
 \|\widetilde{Q}_t^i\,(X_i^\dagger\,\eta_i)\|_2^2
 \bigr] \right]
\;\le\;
\sum_{i=1}^{k-1}
\mathbb{E} \bigl\|X_i^\dagger\|^2\,
 \|\eta_i\|_2^2
\;\;\le\;\;
C\sum_{i=1}^{k-1}
\mathbb{E} \|\eta_i\|_2^2,
\]
for some uniform constant \(C=\max_i\|X_i^\dagger\|^2\). Taking the expectation with
\(\mathbb{E}[\|\eta_i\|_2^2] = \sigma^2\) yields the bound 
\begin{align}
\label{eq:noise_sum_bound} 
\sum_{t=1}^k \mathbb{E}\|\!B_t\|^2 \le C(k-1) \sigma^2.
\end{align}

As the last step, we combine previous bounds and prove the theorem.
Continuing from \eqref{eq:bound_theorem3} with bounds \eqref{eq:noiseless_bound_theorem3} and \eqref{eq:noise_sum_bound} yields
\[
\sum_{t=1}^k \mathbb{E}\|X_t\,\epsilon_{t-1}\|^2 \le 2\sum_{t=1}^k \|A_t\|^2 + 2\sum_{t=1}^k \mathbb{E}\|\!B_t\|^2  \le 2\|w^*\|^2 + 2C(k-1)\sigma^2.
\]

The residual is $r_t = X_tw_{t-1} - y_t = -X_t\epsilon_{t-1} - \eta_t$. Due to the independence of $\eta_t$ and $\epsilon_{t-1}$, the expectation of the cross-term vanishes as below.
\[
\mathbb{E}\|r_t\|^2 = \mathbb{E}\|X_t\epsilon_{t-1}\|^2 + \mathbb{E}\|\eta_t\|^2 = \mathbb{E}\|X_t\epsilon_{t-1}\|^2 + \sigma^2.
\]
Summing this over $t=1, \dots, k$ gives
\[
\sum_{t=1}^k \mathbb{E}\|r_t\|^2 = \sum_{t=1}^k \mathbb{E}\|X_t\epsilon_{t-1}\|^2 + k\sigma^2.
\]
Then, we substitute the bound for the first term on the right as follows.
\[
\sum_{t=1}^k \mathbb{E}\|r_t\|^2 \le \Bigl( 2\|w^*\|^2 + 2C(k-1)\sigma^2 \Bigr) + k\sigma^2 = 2\|w^*\|^2 + \bigl((2C+1)k - 2C\bigr)\sigma^2,
\]
so that 
\[
\frac{1}{k}\sum_{t=1}^k \mathbb{E}\|r_t\|^2 \le \frac{2\|w^*\|^2}{k} + \left( 2C+1 - \frac{2C}{k} \right)\sigma^2.
\]
Hence, we obtain the final bound:
\[
\frac{1}{k}\sum_{t=1}^k \mathbb{E}\|X_tw_{t-1} - y_t\|^2 \le \frac{2\,\|w^*\|^2}{k} + (2C+1)\mathcal{O}(\sigma^2).
\]
This completes the proof.
\end{proof}

\section{Generalized Kaczmarz with Step Size
\texorpdfstring{$\lambda\in(0,2)$}{lambda in (0,2)}}

We now study the generalized Kaczmarz algorithm with step size or relaxation parameter \(\lambda\), whose update equation is given by
\[
w_t = w_{t-1} + \lambda\,X_t^\dagger\Bigl( y_t - X_tw_{t-1}\Bigr).
\]
When \(\lambda\in(0,1)\) we speak of \emph{under-relaxation}, and when \(\lambda\in(1,2)\) of \emph{over-relaxation}. Under-relaxation is more conservative (often more stable), whereas over-relaxation may reduce iteration counts in some settings but can also overshoot and is often not uniformly beneficial; importantly, such apparent speedups typically happen only in a finite time and deal with a different quantity than the average regret, and therefore do not contradict our results. \cite{strohmer2009randomized, oswald2015convergence, nikazad2024choosing}. While relaxed Kaczmarz is classical, its performance under a regret lens is comparatively less explored; in what follows we develop operator-theoretic regret bounds with explicit \(\lambda\)-dependence, and we later extend the analysis to noisy observations.

Another recent related work \cite{gunturk2019unrestricted} studies unrestricted iterations of relaxed projections drawn from a finite family in the noiseless case, proving absolute summability of the displacement sequence, root–exponential tail bounds under a regularity assumption, and the impossibility of uniform iteration–rate guarantees across \(\lambda\in(0,2)\). By contrast, our infinite–dimensional operator/block Kaczmarz setting for online prediction -- allowing possibly infinitely many projections (e.g., the Fourier exponential functions) and including noise—uses an operator–factorization argument and \(\lambda\)-effectiveness to obtain regret-based guarantees. These include a sharp noiseless identity with the explicit constant \(1/[\lambda(2-\lambda)]\), dimension–free \(O(1/k)\) average–regret bounds with explicit noise tradeoffs that inform the choice of \(\lambda\), and an inner–function criterion ensuring the effectiveness for exponential systems.

\subsection{Preliminaries on the Generalized Kaczmarz operators}

Let \(\lambda\in (0,2)\) and let \(\{P_j\}_{j\ge 1}\) be a sequence of self-adjoint projections on a Hilbert space $\mathcal{H}$ including $\{X_j^\dagger X_j\}$. We now analyze the properties of this sequence using a canonical index $n$.

We define the operators by
\begin{align}
\widetilde{T}_0 &= I, \\ 
\widetilde{T}_n &= (I-\lambda P_n)(I-\lambda P_{n-1})\cdots (I-\lambda P_1), \quad n \ge 1, \label{eq:Tn_1based}\\[1mm]
\widetilde{Q}_n &= \lambda\,P_n\,\widetilde{T}_{n-1}, \quad n \ge 1. \label{eq:Qn_1based}
\end{align}
Note that $\widetilde{Q}_1 = \lambda P_1 \widetilde{T}_0 = \lambda P_1$.
For $n \ge 1$, a short computation shows that
\[
\widetilde{T}_n = (I-\lambda P_n)\widetilde{T}_{n-1} = \widetilde{T}_{n-1} - \lambda P_n\,\widetilde{T}_{n-1}.
\]

Thus, using the definition \eqref{eq:Qn_1based}, we obtain for $n \ge 1$:
\[
\widetilde{T}_n = \widetilde{T}_{n-1} - \widetilde{Q}_n \quad \text{or} \quad \widetilde{T}_{n-1} - \widetilde{T}_n = \widetilde{Q}_n.
\]
This leads to the following telescoping sum.
\[
\sum_{j=1}^n (\widetilde{T}_{j-1} - \widetilde{T}_j) = \sum_{j=1}^n \widetilde{Q}_j,
\]
which gives
\[
\widetilde{T}_0 - \widetilde{T}_n = \sum_{j=1}^n \widetilde{Q}_j
\]
Since $\widetilde{T}_0 = I$, we have
\[
I-\widetilde{T}_n = \sum_{j=1}^{n} \widetilde{Q}_j.
\]

Now, we are ready to state the following important lemma, generalizing Theorem 3 in \cite{jorgensen2020kaczmarz}.

\begin{lemma}
\label{lem:Generalized_Kaczmarz_Scaled_Parseval}
(i) For any $N \ge 1$:
\[ I - \widetilde{T}_N^* \widetilde{T}_N = \frac{2 - \lambda}{\lambda} \sum_{n=1}^N \widetilde{Q}_n^* \widetilde{Q}_n. \]

In parts (ii) and (iii), assume that the sequence of operators $\widetilde{T}_n$ converges in the strong operator topology (SOT) to $0$ as $n \to \infty$ (i.e., the system is $\lambda$-effective by Definition~\ref{def:lambda effectiveness}). Then, we have
\begin{enumerate}
\item[(ii)] The sum $\sum\limits_{j=1}^N \widetilde{Q}_j^* \widetilde{Q}_j$ converges in SOT as $N \to \infty$ to $\frac{\lambda}{2-\lambda} I$.
\item[(iii)] The sum $\sum\limits_{j=1}^N \widetilde{Q}_j$ converges in SOT as $N \to \infty$ to $I$.
\end{enumerate}
\end{lemma}

\begin{proof}
(iii) is an immediate consequence of $I-\widetilde{T}_n = \sum_{j=1}^{n} \widetilde{Q}_j$ and the assumption $\widetilde{T}_n \rightarrow 0$ in SOT.

\paragraph{Proof of (i)}

We adapt the proof technique from Theorem 3 in \cite{jorgensen2020kaczmarz}, generalizing for all $\lambda \in (0, 2)$.

\medskip

First, we start with expanding $\widetilde{T}_n^* \widetilde{T}_n$.

For $n \ge 1$, the definition is $\widetilde{T}_n = (I - \lambda P_n) \widetilde{T}_{n-1}$. Taking the adjoint, using the self-adjointness of $P_n$ and $\lambda \in \R$, we have
\[ \widetilde{T}_n^* = \widetilde{T}_{n-1}^* (I - \lambda P_n)^* = \widetilde{T}_{n-1}^* (I - \lambda P_n). \]
Now we compute the product $\widetilde{T}_n^* \widetilde{T}_n$:
\begin{align*}
\widetilde{T}_n^* \widetilde{T}_n &= \left( \widetilde{T}_{n-1}^* (I - \lambda P_n) \right) \left( (I - \lambda P_n) \widetilde{T}_{n-1} \right)
= \widetilde{T}_{n-1}^* (I - \lambda P_n)^2 \widetilde{T}_{n-1}.
\end{align*}
From $P_n^2 = P_n$, we have $(I - \lambda P_n)^2 = I - (2\lambda - \lambda^2) P_n$ so that 
\begin{align}
\nonumber
\widetilde{T}_n^* \widetilde{T}_n &= \widetilde{T}_{n-1}^* (I - (2\lambda - \lambda^2) P_n) \widetilde{T}_{n-1} \\
&= \widetilde{T}_{n-1}^* \widetilde{T}_{n-1} - (2\lambda - \lambda^2) \widetilde{T}_{n-1}^* P_n \widetilde{T}_{n-1}. \label{eq:TnT_recursion_step1}
\end{align}

Next, we express the recursion for $\widetilde{T}_n^* \widetilde{T}_n$ in $\widetilde{Q}_n$.

Recall the definition $\widetilde{Q}_n = \lambda P_n \widetilde{T}_{n-1}$ for $n \ge 1$. Its adjoint is $\widetilde{Q}_n^* = \lambda \widetilde{T}_{n-1}^* P_n$.
Consider the product $\widetilde{Q}_n^* \widetilde{Q}_n$:
\begin{align*}
\widetilde{Q}_n^* \widetilde{Q}_n = (\lambda \widetilde{T}_{n-1}^* P_n) (\lambda P_n \widetilde{T}_{n-1}) = \lambda^2 \widetilde{T}_{n-1}^* P_n^2 \widetilde{T}_{n-1} = \lambda^2 \widetilde{T}_{n-1}^* P_n \widetilde{T}_{n-1} \quad (\text{since } P_n^2 = P_n), 
\end{align*}
and therefore, 
\[ \widetilde{T}_{n-1}^* P_n \widetilde{T}_{n-1} = \frac{1}{\lambda^2} \widetilde{Q}_n^* \widetilde{Q}_n. \]
We substitute this expression back into Eq. \eqref{eq:TnT_recursion_step1} as follows.
\begin{align*}
\widetilde{T}_n^* \widetilde{T}_n &= \widetilde{T}_{n-1}^* \widetilde{T}_{n-1} - (2\lambda - \lambda^2) \left( \frac{1}{\lambda^2} \widetilde{Q}_n^* \widetilde{Q}_n \right) = \widetilde{T}_{n-1}^* \widetilde{T}_{n-1} - \frac{2 - \lambda}{\lambda} \widetilde{Q}_n^* \widetilde{Q}_n, \;\; \text{for $n \ge 1$.}
\end{align*}

Hence, 
\[ \sum_{n=1}^N \left( \widetilde{T}_{n-1}^* \widetilde{T}_{n-1} - \widetilde{T}_n^* \widetilde{T}_n \right) = \sum_{n=1}^N \frac{2 - \lambda}{\lambda} \widetilde{Q}_n^* \widetilde{Q}_n. \]
The left side is a telescoping sum which evaluates to $\widetilde{T}_0^* \widetilde{T}_0 - \widetilde{T}_N^* \widetilde{T}_N = I - \widetilde{T}_N^* \widetilde{T}_N$.
Thus, for any $N \ge 1$:
\[ I - \widetilde{T}_N^* \widetilde{T}_N = \frac{2 - \lambda}{\lambda} \sum_{n=1}^N \widetilde{Q}_n^* \widetilde{Q}_n. \]
This proves Part (i) of the lemma.

\paragraph{Proof of (ii)}

As stated in the lemma, we assume $\widetilde{T}_N \to 0$ (SOT).
Since $\widetilde{T}_N$ converges strongly, the sequence of operator norms $\{\|\widetilde{T}_N\|\}$ is uniformly bounded. Let $M = \sup_N \|\widetilde{T}_N\| < \infty$. Using operator norm properties and $\|\widetilde{T}_N^*\| = \|\widetilde{T}_N\|$:
\[
\| (\widetilde{T}_N^* \widetilde{T}_N) x \| = \| \widetilde{T}_N^* (\widetilde{T}_N x) \| \le \|\widetilde{T}_N^*\| \|\widetilde{T}_N x\| = \|\widetilde{T}_N\| \|\widetilde{T}_N x\|.
\]
Applying the uniform bound $M$:
\[
\| (\widetilde{T}_N^* \widetilde{T}_N) x \| \le M \|\widetilde{T}_N x\|.
\]
Since $\lim_{N\to\infty} \|\widetilde{T}_N x\| = 0$, we take the limit as $N \to \infty$:
\[
\lim_{N\to\infty} \| (\widetilde{T}_N^* \widetilde{T}_N) x \| \le \lim_{N\to\infty} (M \|\widetilde{T}_N x\|) = 0.
\]
Hence, we have $\lim_{N\to\infty} \| (\widetilde{T}_N^* \widetilde{T}_N) x \| = 0$, which means
\[
\lim_{N\to\infty} \widetilde{T}_N^* \widetilde{T}_N = 0 \quad (\text{SOT}).
\]
Now, we take the limit $N \to \infty$ in the identity of Part (i) of the lemma to obtain
\[
\lim_{N\to\infty} (I - \widetilde{T}_N^* \widetilde{T}_N) = \lim_{N\to\infty} \frac{2 - \lambda}{\lambda} \sum_{n=1}^N \widetilde{Q}_n^* \widetilde{Q}_n \quad (\text{in SOT}).
\]
Since $\lim_{N\to\infty} \widetilde{T}_N^* \widetilde{T}_N = 0 \; (\text{in SOT})$, we have
\[
I = \frac{2 - \lambda}{\lambda} \left( \lim_{N\to\infty} \sum_{n=1}^N \widetilde{Q}_n^* \widetilde{Q}_n \right) \quad \text{or} \quad
\sum_{n=1}^\infty \widetilde{Q}_n^* \widetilde{Q}_n = \frac{\lambda}{2 - \lambda} I \quad (\text{SOT}),
\] where $\sum_{n=1}^\infty \widetilde{Q}_n^* \widetilde{Q}_n$ denote the SOT limit.
This shows that $\lambda$-effectiveness implies the SOT convergence of the sum to $\frac{\lambda}{2-\lambda} I$.
This proves Part (ii) of the lemma.
\end{proof}

\begin{corollary}[Scaled Parseval-type identity]
\label{cor:scalar_parseval_identity}
If the system $\{P_j\}_{j\ge 1}$ is \(\lambda\)–effective for $\lambda \in (0, 2)$, then the following identity holds for all $x \in \mathcal{H}$:
\[
\sum_{j=1}^{\infty} \|\widetilde{Q}_j x\|^2 = \frac{\lambda}{2-\lambda} \|x\|^2.
\]
\end{corollary}

\begin{proof}
This follows directly by taking the inner product with $x$ to $\sum_{n=1}^\infty \widetilde{Q}_n^* \widetilde{Q}_n = \frac{\lambda}{2 - \lambda} I \; (\text{SOT})$:
\[
\sum_{j=1}^{\infty} \|\widetilde{Q}_j x\|^2 = \sum_{j=1}^{\infty} \langle \widetilde{Q}_j x, \widetilde{Q}_j x \rangle = \left\langle \left(\sum_{j=1}^\infty \widetilde{Q}_j^* \widetilde{Q}_j \right) x, x \right\rangle = \left\langle \left(\frac{\lambda}{2-\lambda} I\right) x, x \right\rangle = \frac{\lambda}{2-\lambda} \|x\|^2.
\]
\end{proof}

Also, recall that by definition for $n \ge 1$,
\begin{align}
\label{eq:Q_1}
\widetilde{Q}_n = \lambda P_n\,\widetilde{T}_{n-1}. \quad \text{and} \quad \widetilde{Q}_1 = \lambda P_1 \widetilde{T}_0 = \lambda P_1 \; (\text{since } \widetilde{T}_0 = I) 
\end{align}

Because $\widetilde{T}_{n-1} = I - \sum_{j=1}^{n-1} \widetilde{Q}_j$ for $n \ge 2$,
\begin{align}
\label{eq:Q_n}
\widetilde{Q}_n = \lambda P_n \Bigl(I-\sum_{j=1}^{n-1}\widetilde{Q}_j\Bigr), \; \; \text{$n \ge 2$}.
\end{align}

Suppose that the projections are of rank one, i.e., there exist unit vectors $e_j$ such that
\[
P_j = e_j e_j^*,\quad \|e_j\|_2=1,\quad \text{for } j\ge 1.
\] Then, since $\widetilde{Q}_j$ is of rank one as well with range in span of $\{e_j\}$, there exists a unique vector $h_j^{(\lambda)} \in \mathcal{H}$ such that $\widetilde{Q}_j x = e_j\langle h_j^{(\lambda)}, x \rangle$ (by Riesz representation theorem). This sequence $\{h_j^{(\lambda)}\}$ is called the \emph{auxiliary sequence} associated with $\{e_j\}$ for the generalized Kaczmarz method with relaxation parameter $\lambda$.

\begin{corollary}[Rank--one case]
\label{cor:rank1_1based}
Consider the auxiliary sequence $h_j^{(\lambda)}$ associated with $\{e_j\}$ for the generalized Kaczmarz method with relaxation parameter $\lambda$. Then, \eqref{eq:Q_1} and \eqref{eq:Q_n} can be expressed in terms of $h_j^{(\lambda)}$ as follows:
\[
h_1^{(\lambda)} = \lambda\,e_1,\qquad
h_n^{(\lambda)} = \lambda\Bigl(e_n-\sum_{k=1}^{n-1}\overline{\langle e_n,e_k\rangle}\,h_k^{(\lambda)}\Bigr),\quad n\ge 2.
\]
Also, under \(\lambda\)–effectiveness, strong convergence holds:
\[
x = \sum_{j=1}^{\infty} \widetilde{Q}_j x = \sum_{j=1}^{\infty} e_j\langle h_j^{(\lambda)}, x\rangle,\quad \forall\, x\in H.
\]
\end{corollary}

\begin{theorem}[Sharp regret‑identity for the $\lambda$‑Kaczmarz, noiseless case]
\label{thm:sharp_lambda}
Assume
\begin{itemize}
 \item the data are \emph{noiseless}: \(y_t=X_tw^{\!*}\) (no noise in the observation),
 \item the measurement operator $X_t$ is a partial isometry ($X_t^{*}X_t=P_t$),
\item the sequence of projections $\{P_t\}_{t\ge 1}$ is \(\lambda\)-\emph{effective}, meaning
 \[
 \widetilde T_{t} \;=\;(I-\lambda P_t)\cdots(I-\lambda P_1)\;\xrightarrow[t\to\infty]{\text{SOT}} 0,
\]
where \(\widetilde T_{0}=I\).
\end{itemize}
Let \(\{w_t\}_{t\ge 0}\) be the generalized Kaczmarz iterates.
Then, we have
\[
\;
\sum_{t=1}^{\infty}\!
\bigl\|X_tw_{t-1}-y_t\bigr\|^{2}
\;=\;
\frac{1}{\lambda\,(2-\lambda)}\,\|w^{\!*}\|^{2}.
\;
\]
Consequently, for every finite horizon \(k\),
\[
\frac{1}{k}\sum_{t=1}^{k}
\bigl\|X_tw_{t-1}-y_t\bigr\|^{2}
\;\le\;
\frac{1}{\lambda(2-\lambda)\,k}\,\|w^{\!*}\|^{2},
\]
and the constant \(\frac{1}{\lambda(2-\lambda)}\) is the best‑possible (tight) one.
\end{theorem}

\begin{proof}
Because this is the noiseless setting, the update equation for the generalized Kaczmarz iterates is given by \(
 w_t
 =w_{t-1}
 +\lambda P_t\bigl(w^{\!*}-w_{t-1}\bigr),\; \text{with $w_0=0$}.
\) Then, using the error \(\epsilon_{t-1}=w^{\!*}-w_{t-1}\), we have
\[
\epsilon_{t-1}
\;=\;
\widetilde T_{t-1}\,w^{\!*},
\qquad
\text{so}\quad
P_t\epsilon_{t-1}
\;=\;
P_t\widetilde T_{t-1}\,w^{\!*}
\;=\;
\frac{1}{\lambda}\,\widetilde Q_t\,w^{\!*},
\]
where \(\widetilde Q_t=\lambda P_t\widetilde T_{t-1}\) by \eqref{eq:Q_1}.

Hence the residual on round \(t\) is
\(
r_t
=X_tw_{t-1}-y_t
=X_t(w_{t-1}-w^{\!*})
=-\,X_t\epsilon_{t-1}.
\)
From the assumption that \(X_t\) is a partial isometry with \(X_t^{*}X_t=P_t\), we have
\[
\|r_t\|^{2}
=\|X_t\epsilon_{t-1}\|^{2}
=\langle X_t^{*}X_t\epsilon_{t-1},\,\epsilon_{t-1}\rangle
=\langle P_t\epsilon_{t-1},\,\epsilon_{t-1}\rangle
=\|P_t\epsilon_{t-1}\|^{2}
=\frac{1}{\lambda^{2}}\bigl\|\widetilde Q_tw^{\!*}\bigr\|^{2}.
\]
Since $\{P_t\}_{t\ge 1}$ is \(\lambda\)-effective, Corollary~\ref{cor:scalar_parseval_identity} applied to the sequence $(P_t)_{t\ge 1}$ gives
\(
\sum_{t=1}^{\infty}\|\widetilde Q_t\,w^{\!*}\|^{2}
=\frac{\lambda}{2-\lambda}\,\|w^{\!*}\|^{2},
\)
so we obtain
\[
\sum_{t=1}^{\infty}\|r_t\|^{2}
=\frac{1}{\lambda^{2}}
\sum_{t=1}^{\infty}\|\widetilde Q_t\,w^{\!*}\|^{2}
=\frac{1}{\lambda(2-\lambda)}\,\|w^{\!*}\|^{2}.
\]
This is exactly the claimed identity in the theorem. Note that for $\lambda=1$ the identity reduces to $\sum_{t\ge1}\|r_t\|^{2}=\|w^{\!*}\|^{2}$, recovering the sharp case of Theorem~\ref{thm:kaczmarz}; the finite-\(k\) bound follows by truncating the series.
\end{proof}

\begin{remark}[Interpretation of regret in noiseless case as the necessary ``Cost of Learning'']
One might find it counterintuitive that an \textbf{effective} system---one that successfully converges to the solution---is precisely the one that maximizes the total accumulated regret. This seemingly  paradoxical situation can be explained by reframing the total accumulated regret, whose sum of squares is $\sum_{t=1}^{\infty} \norm{X_t w_{t-1} - y_t}^2$, not as a measure of failure, but as the necessary \textbf{cost of learning}.
The quantity $\frac{1}{\lambda(2-\lambda)}\norm{w^*}^2$ can be interpreted as the total ``difficulty budget'' of the problem, representing the amount of ``work'' the Kaczmarz algorithm with step size $\lambda$ must perform to move from the origin $w_0 = 0$ to the true solution $w^*$. The learning process, which consists of a sequence of projections, expends this budget over time.
\begin{itemize}
\item An \textbf{effective} system of tasks, represented by a sequence of projections $\{P_t\}_{t\ge 1}$, is one that is sufficiently thorough to probe the problem from enough different ``angles.'' This forces the algorithm to navigate the full complexity of the parameter space to reveal the underlying truth $w^*$. To successfully complete this comprehensive learning process, the algorithm \emph{must} fully expend its difficulty budget. Thus, as shown in Theorem~\ref{thm:sharp_lambda}, the total accumulated sum of squared residuals necessarily equals $\frac{1}{\lambda(2-\lambda)}\norm{w^*}^2$. A high cumulative regret is therefore a sign that a difficult problem has been solved completely.

\item Conversely, an \textbf{ineffective} system lacks this thoroughness or comprehensiveness. The algorithm may learn quickly from the few perspectives offered and then stagnate, failing to converge to $w^*$ (i.e., $\norm{w_k - w^*} \not\to 0$). Because the learning journey is incomplete, the algorithm does not expend the full difficulty budget, and its total accumulated regret is strictly less than the maximum possible value. In this context, a low total regret is an evidence of the algorithm's failure in the noiseless setting.
\end{itemize}
In essence, an algorithm that learns from a limited set of only easy or superficial tasks may appear to have low cumulative regret, but it does so at the cost of failing its primary objective --- learning the truth $w^*$. True convergence requires the problem's full complexity, and the total regret is simply a measure of this productive struggle. See \cite{hazan2016introduction, russo2018tutorial} for a similar interpretation of regret type quantities.
\end{remark}

\subsection{Noisy Kaczmarz regret bound with step size
\texorpdfstring{$\lambda\in(0,2)$}{lambda in (0,2)}}

In this section, we consider the noisy Kaczmarz iteration with step size \(\lambda\). Let \(\{w_t\}_{t=0}^k\) be defined by
\[
w_t = w_{t-1} + \lambda\,X_t^\dagger\Bigl( y_t - X_tw_{t-1}\Bigr),\qquad w_0=0,
\]
with observations
\[
y_t = X_tw^* + \eta_t,
\]
and, as in the noisy regular Kaczmarz case, suppose that the noise vectors \(\{\eta_t\}\) are i.i.d.\ with zero mean and variance $\sigma^2$ (so, $\mathbb{E}\|\eta_t\|^2 = \sigma^2$).

\medskip

\begin{theorem}[Noisy Kaczmarz regret bound with step size \(\lambda\)]
\label{thm:noisy_lambda_corrected}
Under the above conditions, Assumption~\ref{asmp:task_bound}, and Assumption~\ref{asmp:pinv} with $\sup_{t} \|X_t^\dagger\|^2 = C < \infty$,
the average regret satisfies
\begin{align*}
\frac{1}{k}\sum_{t=1}^k \mathbb{E}\Bigl[\|X_tw_{t-1} - y_t\|^2\Bigr]
&\le \frac{2\,\|w^*\|^2}{\lambda(2-\lambda)\,k} + \left( \frac{2\lambda C}{2-\lambda} + 1 \right)\sigma^2 \\
&= \frac{2\,\|w^*\|^2}{\lambda(2-\lambda)\,k} + \mathcal{O}(\sigma^2).
\end{align*}

That is, the average regret consists of a decaying term proportional to \(1/k\) (with constant scaling \(1/(\lambda(2-\lambda))\) depending $\lambda$) plus a noise contribution of order \(\sigma^2\).
\end{theorem}

\medskip

\begin{proof}[Proof of Theorem \ref{thm:noisy_lambda_corrected}]
We proceed in four steps, applying a number of identities from the preliminaries.

\medskip

\textbf{Step 1. Error decomposition.}\\[1mm]
Define the error \(\epsilon_t = w^* - w_t\). The update rule is given by
\[
w_t = w_{t-1} + \lambda\,X_t^\dagger\Bigl( X_tw^* + \eta_t - X_tw_{t-1}\Bigr).
\]
Letting \(P_t = X_t^\dagger X_t\) as before, the error recursion becomes
\[
\epsilon_t = w^* - w_t = w^* - \left( w_{t-1} + \lambda\,X_t^\dagger(X_tw^* + \eta_t - X_tw_{t-1}) \right),
\]
which simplifies to
\[
\epsilon_t = \Bigl(I-\lambda\,P_t\Bigr)\epsilon_{t-1} - \lambda\,X_t^\dagger\,\eta_t.
\]

Define the \emph{noiseless component}
\[
A_t = X_t\Bigl[(I-\lambda\,P_{t-1})\cdots (I-\lambda\,P_1)\Bigr]w^*,
\]
and the \emph{noise component}
\[
B_t = X_t\sum_{i=1}^{t-1} \Bigl[(I-\lambda\,P_{t-1})\cdots (I-\lambda\,P_{i+1})\Bigr]\lambda\,X_i^\dagger\,\eta_i.
\]

We unroll the recursion $\epsilon_t = \Bigl(I-\lambda\,P_t\Bigr)\epsilon_{t-1} - \lambda\,X_t^\dagger\,\eta_t$ and multiply by $X_t$ to obtain
\[
X_t\,\epsilon_{t-1} = A_t - B_t.
\]

Then, by the inequality \( \|u-v\|^2\le 2\|u\|^2+2\|v\|^2 \), we have
\[
\|X_t\epsilon_{t-1}\|^2 \le 2\|A_t\|^2 + 2\|B_t\|^2.
\]

\medskip

\textbf{Step 2. Bound on the noiseless component.}\\[1mm]
As before, we have
\begin{align*}
\|A_t\|^2 &= \Bigl\|X_t(I-\lambda\,P_{t-1})\cdots (I-\lambda\,P_1)w^*\Bigr\|^2 \\
&\le \Bigl\|P_t(I-\lambda\,P_{t-1})\cdots (I-\lambda\,P_1)w^*\Bigr\|^2,
\end{align*} since \(\|X_t u\|\le \|P_t u\|\), a consequence of Assumption~\ref{asmp:task_bound}.
Since $
\widetilde{Q}_t = \lambda\,P_t(I-\lambda P_{t-1})\cdots (I-\lambda P_1)$, we have
\[
P_t(I-\lambda\,P_{t-1})\cdots (I-\lambda\,P_1)w^* = \frac{1}{\lambda}\widetilde{Q}_t w^*.
\]
Hence,
\[
\|A_t\|^2 \le \frac{1}{\lambda^2}\|\widetilde{Q}_t w^*\|^2.
\]
Therefore, the sum over the first \(k\) terms and applying Part (i) of Lemma~\ref{lem:Generalized_Kaczmarz_Scaled_Parseval} yield,

\[
\sum_{t=1}^k \|\widetilde{Q}_t w^*\|^2 \le \frac{\lambda}{2-\lambda}\|w^*\|^2.
\]

Hence, 
\[
\sum_{t=1}^k \|A_t\|^2 \le \frac{1}{\lambda^2} \sum_{t=1}^k \|\widetilde{Q}_t w^*\|^2 \le \frac{1}{\lambda(2-\lambda)}\|w^*\|^2.
\]

\medskip

\textbf{Step 3. Bound on the noise component.}\\[1mm]
For \(1\le i < t\le k\), we bound \(B_t\). By definition,
\[
B_t = X_t\sum_{j=1}^{t-1} \Bigl[(I-\lambda\,P_{t-1})\cdots (I-\lambda\,P_{j+1})\Bigr]\lambda\,X_j^\dagger\,\eta_j.
\]
Using \(\|X_t u\|\le \|P_t u\|\), we have
\[
\|B_t\| \le \Bigl\|\;P_t\sum_{j=1}^{t-1} \Bigl[(I-\lambda\,P_{t-1})\cdots (I-\lambda\,P_{j+1})\Bigr]\lambda\,X_j^\dagger\,\eta_j\Bigr\|.
\]
For convenience, define the \emph{partial noise operators} relative to the index \(j\) as
\[
\widetilde{Q}_t^j := \lambda\,P_t\Bigl[(I-\lambda\,P_{t-1})\cdots (I-\lambda\,P_{j+1})\Bigr],\quad 1\le j < t.
\]
Then,
\[
\|B_t\| \le \Bigl\|\sum_{j=1}^{t-1} \widetilde{Q}_t^j \bigl(X_j^\dagger\,\eta_j\bigr)\Bigr\|.
\]
Taking squares and expectations (using independence and zero-mean of \(\eta_j\), assuming \(\tau\) is independent of \(\eta\), so cross-terms vanish), we obtain
\[
\mathbb{E}\|B_t\|^2 \le \mathbb{E}\Biggl[\sum_{j=1}^{t-1} \|\widetilde{Q}_t^j (X_j^\dagger\,\eta_j)\|^2\Biggr] = \sum_{j=1}^{t-1} \mathbb{E}\Bigl[\|\widetilde{Q}_t^j (X_j^\dagger\,\eta_j)\|^2\Bigr].
\]
Summing over \(t=1,\dots,k\) followed by switching the double summation gives
\[
\sum_{t=1}^k \mathbb{E}\|B_t\|^2 \le \sum_{t=2}^k \sum_{j=1}^{t-1} \mathbb{E}\Bigl[\|\widetilde{Q}_t^j (X_j^\dagger\,\eta_j)\|^2\Bigr] = \sum_{j=1}^{k-1} \sum_{t=j+1}^k \mathbb{E}\Bigl[\|\widetilde{Q}_t^j (X_j^\dagger\,\eta_j)\|^2\Bigr].
\]
Now consider the inner sum for a fixed \(j\). Recall that the operators \(\widetilde{Q}_t^j\) for \(t > j\) correspond to applying \(\lambda P_t\) followed by products of \((I-\lambda P)\) starting from index \(t-1\) down to \(j+1\), as in the previous section.

Therefore, by Part (i) in Lemma~\ref{lem:Generalized_Kaczmarz_Scaled_Parseval} again, we have
\[
\sum_{t=j+1}^k \|\widetilde{Q}_t^j (X_j^\dagger\,\eta_j)\|^2 \le \frac{\lambda}{2-\lambda}\|X_j^\dagger\,\eta_j\|^2.
\]
After substituting this bound into the sum for \(\mathbb{E}\|B_t\|^2\), we have
\[
\sum_{t=1}^k \mathbb{E}\|B_t\|^2 \le \sum_{j=1}^{k-1} \mathbb{E}\Biggl[ \sum_{t=j+1}^k \|\widetilde{Q}_t^j (X_j^\dagger\,\eta_j)\|^2 \Biggr] \le \sum_{j=1}^{k-1} \mathbb{E}\Biggl[ \frac{\lambda}{2-\lambda}\|X_j^\dagger\,\eta_j\|^2 \Biggr].
\]
Since \( C = \max_l \|X_l^\dagger\|^2 \) by Assumption~\ref{asmp:pinv} and \(\mathbb{E}\|\eta_j\|^2 = \sigma^2\), we have \(\mathbb{E}\|X_j^\dagger\,\eta_j\|_2^2 \le \|X_j^\dagger\|^2\mathbb{E}[ \|\eta_j\|_2^2] = \|X_j^\dagger\|^2 \sigma^2 \le C\sigma^2\). Thus,
\[
\sum_{t=1}^k \mathbb{E}\|B_t\|^2 \le \sum_{j=1}^{k-1} \frac{\lambda}{2-\lambda} C\sigma^2 = \frac{\lambda C (k-1)}{2-\lambda}\,\sigma^2.
\]

\medskip

\textbf{Step 4. Combining the bounds.}\\[1mm]
From Step~1 we have
\[
\|X_t\epsilon_{t-1}\|^2 \le 2\|A_t\|^2 + 2\|B_t\|^2.
\]
Summing over \(t=1,\dots,k\) and taking expectations give
\[
\sum_{t=1}^k \mathbb{E}\|X_t\epsilon_{t-1}\|^2 \le 2\sum_{t=1}^k \mathbb{E}\|A_t\|^2 + 2\sum_{t=1}^k \mathbb{E}\|B_t\|^2 = 2\sum_{t=1}^k \|A_t\|^2 + 2\sum_{t=1}^k \mathbb{E}\|B_t\|^2,
\]
where the equality is from $A_t$ not depending on the noise sequence $\{\eta_t\}$.

Using the bounds from Steps~2 and~3,
\[
\sum_{t=1}^k \mathbb{E}\|X_t\epsilon_{t-1}\|^2 \le \frac{2}{\lambda(2-\lambda)}\|w^*\|^2 + \frac{2\lambda C(k-1)}{2-\lambda}\,\sigma^2.
\]
Denote the residual for task \(t\) by
\[
r_t = X_tw_{t-1} - y_t = X_tw_{t-1} - (X_tw^* + \eta_t) = -X_t\epsilon_{t-1} - \eta_t.
\]
Since \(\epsilon_{t-1}\) depends on \(\eta_1, \dots, \eta_{t-1}\), so the mean-zero i.i.d. noise \(\eta_t\) is independent of \(\epsilon_{t-1}\). Hence, the cross term vanishes in expectation:
\[
\mathbb{E}\|r_t\|^2 = \mathbb{E}\|-X_t\epsilon_{t-1} - \eta_t\|^2 = \mathbb{E}\|X_t\epsilon_{t-1}\|^2 + \mathbb{E}\|\eta_t\|^2 + 2\mathbb{E}[\text{Re}\langle X_t\epsilon_{t-1}, \eta_t \rangle]
\]
\[
\mathbb{E}\|r_t\|^2 = \mathbb{E}\|X_t\epsilon_{t-1}\|^2 + \sigma^2.
\]
Summing over \(t=1,\dots,k\):
\[
\sum_{t=1}^k \mathbb{E}\|r_t\|^2 = \sum_{t=1}^k \mathbb{E}\|X_t\epsilon_{t-1}\|^2 + k \sigma^2.
\]
Substituting the bound derived above:
\[
\sum_{t=1}^k \mathbb{E}\|r_t\|^2 \le \Biggl( \frac{2}{\lambda(2-\lambda)}\|w^*\|^2 + \frac{2\lambda C(k-1)}{2-\lambda}\,\sigma^2 \Biggr) + k \sigma^2.
\]
Dividing by \(k\) yields the final average regret bound:
\[
\frac{1}{k}\sum_{t=1}^k \mathbb{E}\|r_t\|^2 \le \frac{2\|w^*\|^2}{\lambda(2-\lambda)k} + \left( \frac{2\lambda C(1-1/k)}{2-\lambda} + 1 \right)\sigma^2.
\]
For large \(k\), the noise term approaches \(\left( \frac{2\lambda C}{2-\lambda} + 1 \right)\sigma^2\). Thus, we can write the bound as
\[
\frac{1}{k}\sum_{t=1}^k \mathbb{E}\|r_t\|^2 \le \frac{2\,\|w^*\|^2}{\lambda(2-\lambda)\,k} + \mathcal{O}(\sigma^2).
\]
This completes the proof of the theorem.
\end{proof}

\begin{corollary}[Noise–aware choice of the step size $\lambda$]
In the noisy bound of Theorem~\ref{thm:noisy_lambda_corrected}, the $\lambda$–dependence enters through
\[
\frac{1}{k}\sum_{t=1}^k \mathbb{E}\Bigl[\|X_tw_{t-1} - y_t\|^2\Bigr]
\;\lesssim\;
\underbrace{\frac{2\|w^*\|^2}{\lambda(2-\lambda)\,k}}_{\text{learning term}}
\;+\;
\underbrace{\Bigl(\frac{2\lambda C}{2-\lambda}+1\Bigr)\sigma^2}_{\text{noise floor}}.
\]
Ignoring the $\,+1\,$ (which does not depend on $\lambda$), minimizing the surrogate
\[
f(\lambda)\;=\;\frac{a}{\lambda(2-\lambda)}\;+\;\frac{2b\,\lambda}{2-\lambda},
\qquad
a:=\frac{2\|w^*\|^2}{k},\ \ b:=C\sigma^2,
\]
over $\lambda\in(0,2)$ yields a unique global minimizer
\[
\lambda^\star
\;=\;
\frac{\sqrt{a^2+8ab}-a}{4b}
\;=\;
\frac{\sqrt{1+8\beta}-1}{4\beta},
\qquad
\beta:=\frac{b}{a}=\frac{C\sigma^2\,k}{2\|w^*\|^2}.
\]
Hence $\lambda$ trades bias for variance in a simple way governed by the dimensionless ratio $\beta$:
\[
\lambda^\star(\beta)
=\begin{cases}
1-2\beta+O(\beta^2), & \beta\ll 1 \quad\text{(low noise or short horizon): take the classical step } \lambda\approx 1;\\[2mm]
\displaystyle \frac{1}{\sqrt{2}}\beta^{-1/2}+o\!\bigl(\beta^{-1/2}\bigr),
& \beta\gg 1 \quad\text{(high noise or long horizon): under-relax, }\lambda\downarrow 0.
\end{cases}
\]
Two notable implications:
(i) the optimizer always lies in $(0,1]$—our bound never favors over-relaxation under noise;
(ii) as $k$ grows, $\beta$ grows linearly, so the optimal $\lambda$ decays like $k^{-1/2}$.
\end{corollary}

For the benefit of the readers, we have included the following references which cover regret bounds and a priori estimates within several applications in machine learning and in engineering; see \cite{MR4906562, MR4883940, MR4635158, MR4531092, MR4211383, MR4930266,MR4855849,MR4796482,MR4690295,MR4549442,MR4396309,MR4066312} and the papers cited there.

\section{Generalized Kaczmarz Algorithm and $\lambda$-Effectiveness}

In this section we consider the noiseless generalized Kaczmarz iteration with relaxation parameter $\lambda\in(0,2)$ given by
\[
w_t = w_{t-1} + \lambda\,X_t^\dagger\Bigl(X_tw^* - X_tw_{t-1}\Bigr), \qquad w_0 = 0, \quad \text{for } t \ge 1.
\]
Defining the projections $P_t := X_t^\dagger X_t$, the update can be written as
\[
w_t = w_{t-1} + \lambda\,P_t\bigl(w^* - w_{t-1}\bigr).
\]
We now transition from the algorithm's time index $t$ to a canonical index $n$ to analyze the properties of the sequence of projections $\{P_n\}_{n\ge 1}$ itself. We assume that $P_n$ are rank-one, i.e.,
\[
P_n = e_n e_n^*, \quad \|e_n\|=1, \quad \text{for } n \ge 1.
\]
Recall that we define the full operators corresponding to this sequence as
\[
\widetilde{T}_0 = I, \qquad \widetilde{T}_n = (I-\lambda P_n)(I-\lambda P_{n-1})\cdots (I-\lambda P_1), \quad n \ge 1,
\]
and
\[
\widetilde{Q}_n = \lambda\,P_n\,\widetilde{T}_{n-1}, \quad n \ge 1.
\]

Since the error after $n$ steps is defined by $\epsilon_n = w^* - w_n$, starting from $\epsilon_0 = w^* - w_0 = w^*$, the error recursion $\epsilon_n = (I - \lambda P_n) \epsilon_{n-1}$ yields:
\[
\epsilon_n = (I-\lambda P_n)(I-\lambda P_{n-1})\cdots (I-\lambda P_1) \epsilon_0 = \widetilde{T}_n w^*.
\]
If the system is $\lambda$-effective ($\widetilde{T}_n\to 0$ strongly as $n \to \infty$), then $\epsilon_n \to 0$, meaning that $w_n \to w^*$ in the $\ell_2$ norm.

As established previously in Corollary~\ref{cor:rank1_1based}, the operators $\widetilde{Q}_j$ can be expressed as
\[
\widetilde{Q}_j\,w^* = e_j\langle h_j^{(\lambda)},w^*\rangle,
\]
where the auxiliary sequence $\{h_j^{(\lambda)}\}_{j \ge 1}$ is defined recursively by
\begin{align}
\label{eq:h_recursion_0based_section_5}
h_1^{(\lambda)} = \lambda\,e_1,\qquad
h_n^{(\lambda)} = \lambda\Bigl(e_n-\sum_{k=1}^{n-1}\overline{\langle e_n,e_k\rangle}\,h_k^{(\lambda)}\Bigr),\quad n\ge 2.
\end{align}

Note that this recurrence relation is consistent with the counterpart in \cite{jorgensen2020kaczmarz} for $\lambda = 1$ for the auxiliary sequence for the regular Kaczmarz algorithm. There,
\[
h_n^{(\lambda)} = e_n-\sum_{k=1}^{n-1}\langle e_k,e_n\rangle\,h_k^{(\lambda)},\quad n\ge 2,
\] since $\overline{\langle e_n,e_k\rangle} = \langle e_k,e_n\rangle$.
Corollary~\ref{cor:rank1_1based} also states strong convergence for the rank-one case, namely, that $\lambda$-effectiveness implies
\[
w^* = \sum_{j=1}^{\infty} e_j \langle h_j^{(\lambda)}, w^* \rangle.
\]
This gives an expansion of $w^*$ in terms of a possibly non-orthogonal sequence $\{e_j\}_{j\ge 1}$, with coefficients determined by the vectors $\{h_j^{(\lambda)}\}_{j \ge 1}$ as we discuss next.

\subsection*{Impact of primary sequence terms of auxiliary sequence}

We first investigate the impact of the degree of orthogonality of the sequence $\{e_n\}$ to the auxiliary sequence $\{h_n^{(\lambda)}\}$ as we have derived before.
\[
h_1^{(\lambda)} = \lambda\,e_1,\qquad
h_n^{(\lambda)} = \lambda\Bigl(e_n-\sum_{k=1}^{n-1}\overline{\langle e_n,e_k\rangle}\,h_k^{(\lambda)}\Bigr),\quad n\ge 2.
\]

\paragraph{Orthogonal case.}
If the vectors $\{e_j\}_{j\ge 1}$ are pairwise orthogonal, then for $n\ge 2$, we have $\langle e_n, e_k \rangle = 0$ for $1 \le k < n$. Then the recursion simplifies to
\[
h_n^{(\lambda)} = \lambda e_n, \quad \text{for } n \ge 2.
\]
Since $h_1^{(\lambda)} = \lambda e_1$, we have $h_n^{(\lambda)} = \lambda e_n$ for all $n \ge 1$. Thus, in this case, the sequence $\{h_n^{(\lambda)}\}$ associated with the Kaczmarz operators is just a rescaled version of the original sequence $\{e_n\}$ by \(\lambda\).

\paragraph{Non-Orthogonal case.}
More generally, when the $e_j$ are not orthogonal, the recursion involves contributions from the previous terms. Let us compute the first few terms:
\[
h_1^{(\lambda)} = \lambda\,e_1
\]
\[
h_2^{(\lambda)} = \lambda\Bigl(e_2 - \overline{\langle e_2, e_1\rangle}\,h_1^{(\lambda)}\Bigr)
= \lambda e_2 - \lambda \overline{\langle e_2, e_1\rangle} (\lambda e_1)
= \lambda\,e_2 - \lambda^2 \overline{\langle e_2, e_1\rangle}\,e_1
\]
\[
\begin{aligned}
h_3^{(\lambda)} &= \lambda\Bigl(e_3 - \overline{\langle e_3,e_1\rangle}\,h_1^{(\lambda)} - \overline{\langle e_3,e_2\rangle}\,h_2^{(\lambda)}\Bigr)\\[1mm]
&= \lambda e_3 - \lambda \overline{\langle e_3,e_1\rangle} (\lambda e_1) - \lambda \overline{\langle e_3,e_2\rangle} (\lambda e_2 - \lambda^2 \overline{\langle e_2, e_1 \rangle} e_1) \\
&= \lambda\,e_3 - \lambda^2 \overline{\langle e_3,e_1\rangle}\,e_1 - \lambda^2 \overline{\langle e_3,e_2\rangle}\,e_2 + \lambda^3 \overline{\langle e_3,e_2\rangle}\,\overline{\langle e_2,e_1\rangle}\,e_1.
\end{aligned}
\]
In this expansion, every term other than $\lambda e_n$ involves contributions from previous vectors $e_k$ ($k<n$) and is multiplied by extra multiplicative factors of \(\lambda\). In particular, for \(\lambda\in(0,1)\), these factors decay in magnitude with the ``depth" of the recursion, similar to a geometric decay. In other words, even when the \(e_j\) are not orthogonal, the dominant contribution to \(h_n^{(\lambda)}\) often comes from the term \(\lambda\,e_n\), especially for small $\lambda$, and the influence of the earlier terms ($e_1, \dots, e_{n-1}$) is diminished by additional powers of \(\lambda\).
In the next section, in the case of the exponential function, i.e., $e_n(x) = \{ e^{2\pi i n x}\}_{n=0}^\infty$, we provide the exact formula of the expression of $h_n^{(\lambda)}$ in terms of $e_0, e_1, \dots, e_n$ only. This is based on the work \cite{herr2017fourier}.

\section*{Effectiveness of Exponential Systems for All \texorpdfstring{$\lambda\in(0,2)$}{lambda in (0,2)}}

The results in Sections 1–5 show that provided the underlying projection family is \(\lambda\)-effective, both the classical and relaxed Kaczmarz schemes enjoy optimal \(\mathcal{O}(1/k)\) regret and remain robust to stochastic noise. 

Many applications in signal processing and machine learning rely on systems that are non-orthogonal and often redundant. Mathematically, such systems are studied under the heading of frame theory, providing a rich context for the iterative algorithms and a priori estimates considered here; see \cite{MR3495345, MR3061703, MR4696783, MR3793303, MR4633118, MR2735759, MR2085421, MR4696781} and the papers cited there. A prototypical example is the exponential system \(\{e^{2\pi i n x}\}_{n\ge0}\). This system has been proved effective for the regular Kaczmarz algorithm ($\lambda=1$) if and only if either the underlying measure is singular or is identical to the normalized Lebesgue measure on $\mathbb{T}$. Under exactly which conditions the exponential system remains $\lambda$-effective for every relaxation parameter \(\lambda\in(0,2)\) has been open. Characterizing the conditions which guarantee $\lambda$-effectiveness for every relaxation parameter \(\lambda\in(0,2)\) is the problem we consider next. 

Resolving this question is essential. Without effectiveness, the iterates of the generalized Kaczmarz algorithm, \(w_k\), need not converge to the true solution \(w^*\). A low or vanishing regret in such a scenario would be misleading, as the model would have failed its primary learning objective. The condition of effectiveness is therefore essential, as it ensures that our regret bounds are meaningful performance measures for an algorithm that successfully converges to the true solution, along with the tightness guarantee of our regret bounds in the noiseless case. 

In the rest of the paper, leveraging an inner-function criterion for \(\lambda\)-effectiveness, we prove that the exponential system \(\{e^{2\pi i n x}\}_{n\ge0}\) is \emph{\(\lambda\)-effective for every relaxation parameter in $(0,2)$} if and only if the underlying measure is singular unless $\lambda = 1$ (in which case allows the normalized Lebesgue measure as well). This result generalizes the result of Herr, Weber, and Jorgensen et al. for \(\lambda=1\) \cite{herr2017fourier, jorgensen2020kaczmarz, herr2023fourier} and it shows that the favorable convergence properties of Fourier-type Kaczmarz updates persist both under- and over-relaxation. 
Note that the exponential function $e^{2\pi i n x}$ is rank-one and a unit vector, so it is a partial isometry by Remark~\ref{rmk:rank-one-partial-isometry}. This fact, combined with the $\lambda$-effectiveness of this system (proven in Theorem \ref{thm:lambda_effectiveness_exponential_functions}) satisfies the conditions of Theorem~\ref{thm:sharp_lambda}. This shows that the regret bound is sharp for the generalized Kaczmarz algorithm with exponential systems \(\{e^{2\pi i n x}\}_{n\ge0}\). 

Our proof is based on the exact computation of the associated inner function in the Hardy space and its equivalence, which also yields a generalized form of the normalized Cauchy transform.

\section{Auxiliary Sequence of the Generalized Kaczmarz Algorithm}

\subsection{Equivalence of recursive and combinatorial definitions of \(\alpha_n^{(\lambda)}\)}

For a Borel probability measure \(\mu\) on \([0,1)\), we define the Fourier--Stieltjes transform or the Fourier coefficients of \(\mu\) for an integer $n$ as
\[
\hat{\mu}(n)=\int_0^1 e^{-2\pi i n x}\,d\mu(x).
\]
We also define the inner product on the Hilbert space of functions on \([0,1)\) or the unit circle $\mathbb{T}$ as below:
\[
\inner{f, g} = \int_{\mathbb{T}} \overline{f(\xi)} g(\xi) d\mu(\xi) = \int_{0}^1 \overline{f(x)} g(x) d\mu(x).
\] Here, we have identified the functions on \([0,1)\) and those on $\mathbb{T}$ in a natural way.

Next, we define two sequences $\alpha_n^{(\lambda)}$ and $\beta_n^{(\lambda) }$
in terms of the relaxation parameter $\lambda$ of the generalized Kaczmarz method and the Fourier coefficients $\hat{\mu}(n)$.

\begin{definition}
\label{def:equivalence_alpha}
Let $\alpha_n^{(\lambda)}$ and $ \beta_n^{(\lambda)}$ be sequences of complex numbers. 
\begin{itemize}
    \item \textbf{Combinatorial sequence $\alpha_n^{(\lambda)}$:}
    
    Let $P_n$ be the set of compositions of $n$, i.e., $P_n = \{(n_1, n_2, \dots, n_k) \;|\; k, n_i \in \mathbb{N}^+,  n_1 + n_2 + \dots + n_k = n\}$ and $l(p)$ be the length of the tuple $p \in P_n$.
    \begin{align} 
    \nonumber
    \alpha_0^{(\lambda)} &= \lambda, \\
    \label{eq:combinatorial_seq}
    \alpha_n^{(\lambda)} &= \sum_{p\in P_n}(-1)^{\ell(p)}\,\lambda^{\ell(p)+1}\prod_{j=1}^{\ell(p)}\widehat{\mu}(p_j), \quad n \ge 1. 
    \end{align}
    \item \textbf{Recursive sequence $ \beta_n^{(\lambda)}$:}
    \begin{align*} \beta_0^{(\lambda)} &= \lambda, \\ \beta_n^{(\lambda)} &= -\lambda \sum_{k=0}^{n-1} \widehat{\mu}(n-k) \beta_k^{(\lambda)}, \quad n \ge 1. \end{align*}
\end{itemize} 
\end{definition}

The following lemma states that the combinatorial definition of $\alpha_n^{(\lambda)}$ generates the identical sequence in the recursive definition of $\beta_n^{(\lambda)}$.

\begin{lemma}
\label{lem:equivalence_alpha}
Consider the sequences $\alpha_n^{(\lambda)}$ and $ \beta_n^{(\lambda)}$ in Definition~\ref{def:equivalence_alpha}. Then, 
\(\alpha_n^{(\lambda)} = \beta_n^{(\lambda)}\) for all \(n \ge 0\).
\end{lemma}
This lemma can be easily derived from the relation $\alpha_n= - \sum_{k=0}^{n-1} \widehat{\mu}(n-k) \alpha_k$ in Lemma 2 of \cite{herr2017fourier} by change of variables such as $\widehat{\mu} \rightarrow \lambda \widehat{\mu}$. It can be more directly shown that the combinatorial sequence $\alpha_n^{(\lambda)}$ and the recursive sequence $ \beta_n^{(\lambda)}$ are identical. For completeness, we include the self-contained proof in the appendix.

\subsection{Relation to combinatorial coefficients}

Possibly non-unique Fourier expansions whose Fourier coefficients are based on the regular Kaczmarz algorithm have been explored in \cite{herr2017fourier}, extended to two dimension in \cite{herr2023fourier}, and generalized further to higher dimensions in \cite{berner2025fourier}.  For more references, see also \cite{hegde2021kaczmarz, herr2020harmonic, MR4932919}. In this section, we generalize Lemma 2 in \cite{herr2017fourier}, relating the sequence $\{h_n^{(\lambda)}\}_{n \ge 0}$ defined in \eqref{eq:h_recursion_0based_section_5} to the combinatorial coefficients in terms of the measure $\mu$. 

Recall that $e_n(x) = \{ e^{2\pi i n x}\}_{n=0}^\infty$ and $\langle e_n,e_j\rangle_{L^2(\mu)} = \widehat{\mu}(n-j)$. The recursion for the auxiliary sequence $h_n^{(\lambda)}$ of the generalized Kaczmarz algorithm is:
\begin{equation} \label{eq:h_recursion_0based_recall}
h_0^{(\lambda)} = \lambda\,e_0,\qquad
h_n^{(\lambda)} = \lambda\Bigl(e_n - \sum_{k=0}^{n-1}\overline{\widehat{\mu}(n-k)}\,h_k^{(\lambda)}\Bigr),\quad n\ge 1.
\end{equation} Note that the sequence start with the index $0$ not $1$ since the exponential function is indexed by $e_n(x) =\{ e^{2\pi i n x}\}_{n=0}^\infty$.
As we have defined for $n\ge 1$,
\[
\alpha_n^{(\lambda)} = \sum_{p\in P_n}(-1)^{\ell(p)}\,\lambda^{\ell(p)+1}\prod_{j=1}^{\ell(p)}\widehat{\mu}(p_j),
\]
where $P_n$ denotes the set of compositions of $n$. As established in Lemma~\ref{lem:equivalence_alpha}, these coefficients satisfy the recursion $\alpha_n^{(\lambda)} = -\lambda \sum_{k=0}^{n-1} \widehat{\mu}(n-k) \alpha_k^{(\lambda)}$ for $n \ge 1$.

\begin{lemma}[Relation between $h_n^{(\lambda)}$ and $\alpha_n^{(\lambda)}$]
\label{lem:h_alpha_relation_0based}
Let $\mu$ be a singular Borel probability measure on $[0,1)$ with Fourier--Stieltjes transform $\widehat{\mu}$. Let $\{\alpha_n^{(\lambda)}\}_{n\ge 0}$ be defined by \eqref{eq:combinatorial_seq} and let $\{h_n^{(\lambda)}\}_{n\ge 0}$ be defined recursively by \eqref{eq:h_recursion_0based_recall}. Then, for all real $\lambda \in (0, 2)$, and every $n\ge 0$, we have 
\begin{align}
\label{lem:relation auxilary and primary}
h_n^{(\lambda)} = \sum_{j=0}^{n}\overline{\alpha_{n-j}^{(\lambda)}}\,e_j.
\end{align}
\end{lemma}

\begin{proof}
We prove the identity by induction on $n$.

\textbf{Base case ($n=0$):}
We need to show that $h_0^{(\lambda)}  = \overline{\alpha_0^{(\lambda)}} e_0$.
By definition $h_0^{(\lambda)} = \lambda e_0$. From the definition of $\alpha$, $\alpha_0^{(\lambda)} = \lambda$. Since $\lambda$ is real, $\overline{\alpha_0^{(\lambda)}} = \lambda$. Thus, the identity holds for $n=0$.

\textbf{Inductive step:} Assume $h_k^{(\lambda)} = \sum_{j=0}^{k}\overline{\alpha_{k-j}^{(\lambda)}}\,e_j$ for $0 \le k < n$ (where $n \ge 1$).
Since 
\[
h_n^{(\lambda)} = \lambda\Bigl(e_n - \sum_{k=0}^{n-1}\overline{\widehat{\mu}(n-k)}\,h_k^{(\lambda)}\Bigr),
\]
substituting the induction hypothesis yields
\[
h_n^{(\lambda)} = \lambda\,e_n - \lambda\sum_{k=0}^{n-1}\overline{\widehat{\mu}(n-k)}\left(\sum_{j=0}^{k}\overline{\alpha_{k-j}^{(\lambda)}}\,e_j\right).
\]
Next, we interchange the summation as follows. 
\begin{align}
\label{eq:h_alpha_relation_0based_identity_1}
h_n^{(\lambda)} = \lambda\,e_n - \lambda\sum_{j=0}^{n-1} e_j \left(\sum_{k=j}^{n-1}\overline{\widehat{\mu}(n-k)}\,\overline{\alpha_{k-j}^{(\lambda)}}\right).
\end{align}
We want to show this equals $\sum_{j=0}^{n}\overline{\alpha_{n-j}^{(\lambda)}}\,e_j$.
After we rewrite the sum on right hand side in \eqref{lem:relation auxilary and primary} by separating the $j=n$ term, we have
\begin{align}
\label{eq:h_alpha_relation_0based_identity_2}
\sum_{j=0}^{n}\overline{\alpha_{n-j}^{(\lambda)}}\,e_j = \overline{\alpha_{0}^{(\lambda)}} e_n + \sum_{j=0}^{n-1}\overline{\alpha_{n-j}^{(\lambda)}}\,e_j. 
\end{align}
Since $\overline{\alpha_{0}^{(\lambda)}} = \lambda$ (as $\lambda$ is real), the sum in \eqref{lem:relation auxilary and primary} can be expressed as 
\begin{align}
\label{eq:h_alpha_relation_0based_identity_3}
\lambda e_n + \sum_{j=0}^{n-1}\overline{\alpha_{n-j}^{(\lambda)}}\,e_j. 
\end{align}
Comparing our expression for $h_n^{(\lambda)}$ with this target form, the $\lambda e_n$ terms match. From Lemma~\ref{lem:equivalence_alpha} and the fact that $\lambda$ is real, 
\[ -\lambda \sum_{k=j}^{n-1}\overline{\widehat{\mu}(n-k)}\,\overline{\alpha_{k-j}^{(\lambda)}} = \overline{\alpha_{n-j}^{(\lambda)}}. \]
Thus, replacing $\overline{\alpha_{n-j}^{(\lambda)}}$ in \eqref{eq:h_alpha_relation_0based_identity_3} with $-\lambda \sum_{k=j}^{n-1}\overline{\widehat{\mu}(n-k)}\,\overline{\alpha_{k-j}^{(\lambda)}}$ shows that \eqref{eq:h_alpha_relation_0based_identity_2} is identical to \eqref{eq:h_alpha_relation_0based_identity_1}. By induction, this proves the lemma. 
\end{proof}

\section{Generalized Kaczmarz Fourier Series  and its Effectiveness}

\subsection{Effectiveness of the Fourier exponential functions for Borel singular measure}
In this subsection, we review that $ \{e_n(x) := e^{2\pi i n x}\}_{n\ge 0}$, the Fourier exponential functions with nonnegative indices are effective for any Borel singular measure, which appeared in \cite{kwapien2001kaczmarz, kwapien2006erratum,  herr2017fourier}. 

Recall that the well-known theorem of F. and M. Riesz  (e.g., see Theorem 17.13 in \cite{rudin1987real}) states the following fact: Suppose $\mu$ is a complex Borel measure on $[0,1)$. If $\int_{0}^1 e^{2\pi i n x} d\mu(x) = 0$ for all natural numbers $n$, then $\mu$ is absolutely continuous with respect to the Lebesgue measure. 

Based on this fact and the uniqueness of the Lebesgue's decomposition theorem, Herr and Weber prove the following lemma in \cite{herr2017fourier}:
\begin{lemma}
\label{lem:singular_Borel_measure_exponential_function}
If $\mu$ is a singular Borel measure on $[0,1)$, then $\{ e_n(x)\}_{n\ge 0}$ is linearly dense in $L^2(\mu)$.
\end{lemma}

\begin{definition}
We say that a sequence $\{ \varphi_n \}_{n=0}^\infty$ in a Hilbert space is \emph{stationary} if $\inner{\varphi_{k+m}, \varphi_{l+m}} = \inner{\varphi_{k}, \varphi_{l}}$ for any nonnegative integers $k,l$ and $m$.
\end{definition}
For a given stationary sequence $\{\varphi_n\}$, let $a_m = \inner{ \varphi_{k+m}, \varphi_k}$ for any positive integer $k$ with $k > -m$. Then by Bochner's theorem, there exists a positive measure $\nu$ on the unit circle $\mathbb{T}$ (called the spectral measure of the stationary sequence $\{\varphi_n\}$) such that 
\[
a_m = \int_{\mathbb{T}} \bar{z}^m \nu(dz) = \int_{0}^1 e^{-2\pi i m x} \nu(dz),
\] for each $m \in \mathbb{Z}$. In our case, since
\[
a_m = \inner{ \varphi_{k+m}, \varphi_k} = \int_{\mathbb{T}} e^{-2\pi i (k+m) x}  e^{2\pi i k x} d\mu(x) = \int_{0}^1 e^{-2\pi i m x} d\mu(x),
\] the spectral measure is $\mu$.

Theorem 2 in \cite{kwapien2001kaczmarz} states as follows:
\begin{theorem}
\label{thm:equivalence_stationary}
A stationary sequence of unit vectors that is linearly dense in a Hilbert space is effective (i.e., $\lambda$-effective for $\lambda = 1$) if and only if its spectral measure either coincides with the normalized Lebesgue measure or is singular with respect to the Lebesgue measure. 
\end{theorem}

Thus, this fact and Lemma~\ref{lem:singular_Borel_measure_exponential_function} imply that the system of exponential functions $\{e^{2\pi i n x}\}_{n\ge 0}$  is effective  for any Borel singular measure on $[0,1)$.

We want to generalize Theorem~\ref{thm:equivalence_stationary} to the $\lambda$-effectiveness for all $\lambda \in (0,2)$.

\begin{remark}
Note that the effectiveness of $\lambda=1$ does not necessarily imply the effectiveness of $\lambda \in (0, 2)$; recall that $T_n = (I- P_n)\cdots (I-P_1)$ and $\widetilde{T}_n = (I-\lambda P_n)\cdots (I-\lambda P_1)$. Consider the example where the range of $P_1$ is orthogonal to the rest of the ranges of $P_i$ for all $i \ge 2$. Then for all nonzero $x$ in the range of $P_1$, $T_n x= (I- P_n)\cdots (I-P_1)x = 0$ but $\widetilde{T}_n x = (I-\lambda P_n)\cdots (I-\lambda P_1) x =  (I-\lambda P_n)\cdots (I-\lambda P_2) (x - \lambda x) = (1 - \lambda) (I-\lambda P_n)\cdots (I-\lambda P_2) x = (1 - \lambda) x \neq 0$ unless $\lambda = 1$. 

However, this counterexample does not exclude the possibility of the Fourier exponential functions $e_n(x) := \{e^{2\pi i n x}\}_{n\ge 0}$ $\lambda$-effective, since we know that these functions may not be orthogonal by the proof of Theorem 2 in \cite{kwapien2001kaczmarz} when the measure $\mu$ is singular with respect to the Lebesgue measure. This means that the effectiveness result for the Fourier exponential functions in \cite{kwapien2001kaczmarz} with $\lambda = 1$ is not directly applicable to show its $\lambda$-effectiveness, so we need to consider a more refined argument. In the next section, we prove the $\lambda$-effectiveness for all $\lambda \in (0,2)$ generalizing Theorem~\ref{thm:equivalence_stationary} utilizing Hardy space theory, especially about an inner function. 
\end{remark}

\subsection{Preliminaries: Hardy Space, inner functions, and the Herglotz theorem}
In this section, we introduce the Hardy Space, inner functions, and the Herglotz Theorem, which form the foundation of our analysis of the $\lambda$-effectiveness of Fourier exponential functions. This material can be found in \cite{rudin1987real, MR255841, katznelson2004introduction, MR770982, herr2019positive}.

\subsubsection{The Hardy space \(H^2(\D)\)}
Let $\D$ be the open unit disk on the complex plane. The Hardy space \(H^2(\D)\) is the Hilbert space of all analytic functions
\[
  f(z)=\sum_{n=0}^\infty c_n\,z^n,\qquad z\in\D,
\]
whose power series coefficients are square-summable:
\[
  \|f\|_{H^2}^2 = \sum_{n=0}^\infty |c_n|^2 < \infty.
\]
Equivalently, \(f \in H^2(\D)\) if and only if
\[
  \sup_{0<r<1}\frac{1}{2\pi}\int_0^{2\pi}|f(r e^{i\theta})|^2\,d\theta < \infty.
\]

\subsubsection{Inner functions}
An \emph{inner function} \(b\) on \(\D\) is an analytic function \(b:\D\to\C\) bounded by 1 (i.e., \(|b(z)| \le 1\) for \(z\in\D\)) that has boundary values of modulus~1 almost everywhere on the unit circle \(\T\). Moreover, the operation of multiplication by \(b\) is an isometry on \(H^2(\D)\):
\[
  \|b\,f\|_{H^2}=\|f\|_{H^2}, \quad\forall\,f \in H^2(\D).
\]

\subsubsection{The Cauchy transform and its properties}

Fix a singular Borel probability measure \(\mu\) on \([0,1)\). Define its Cauchy transform of the constant function \(1\):
\[
  F(z) := \int_0^1 \frac{1}{1 - z\,e^{-2\pi i x}}\,d\mu(x) = \sum_{n=0}^\infty \hat{\mu}(n)\,z^n, \quad z\in\D,
\]
where \(\hat{\mu}(n)=\int_0^1 e^{-2\pi i n x}\,d\mu(x)\) are the Fourier coefficients of \(\mu\). 

Note that under the natural identification of $[0,1)$ with the unit circle $\mathbb{T}$, 
\[ F(z) = \int_{\T}\frac{1}{1 - z\overline{\xi}}\,d\mu(\xi). \]
Let \(w = z e^{-2\pi i x}\). Since \(z\in\D\), we have \(|w|=|z|<1\). The real part of the so called kernel function $\frac{1}{1 - w}$ is calculated as:
\[
  \Re\!\left(\frac{1}{1 - w}\right) = \frac{1 - \Re(w)}{|1 - w|^2}.
\]

We check if this is greater than \(1/2\):
\[
  \frac{1 - \Re(w)}{|1 - w|^2} > \frac{1}{2} \quad\Longleftrightarrow\quad 2(1 - \Re(w)) > |1 - w|^2 = 1 - 2\Re(w) + |w|^2,
\]
which simplifies to \(1 > |w|^2\). Since \(|w|<1\), this inequality holds. Thus, for each fixed \(x\),
\[
  \Re\!\left(\frac{1}{1 - z\,e^{-2\pi i x}}\right) > \frac{1}{2}, \quad z\in\D.
\]

Hence, we have
\[
  \Re F(z) = \int_0^1 \Re\!\left(\frac{1}{1 - z\,e^{-2\pi i x}}\right) d\mu(x) > \int_0^1 \frac{1}{2}\,d\mu(x) = \frac{1}{2}.
\]
This shows \(\Re F(z)>1/2\), which implies the weaker condition \(\Re F(z)>0\). This confirms the result in \cite{herr2017fourier}.
Also, \(F(0) = \int_0^1 d\mu(x) = 1\).

\subsection{Boundary values for a singular measure}

Next, we investigate the behavior of $F(z)$ on the unit circle. The ``value'' of $F$ on the circle is its radial limit, $F(e^{it}) = \lim_{r\to1^-} F(re^{it})$.

A fundamental result from harmonic analysis connects these boundary values to the measure $\mu$. Specifically, for a \textbf{singular} Borel measure $\mu$, the real part of the boundary function is constant almost everywhere, namely, 
\[
\Re F(e^{it}) = \frac{1}{2}, \quad \text{for a.e. $t$ with respect to the Lebesgue measure on $[0,1)$}.
\]

This result is a defining characteristic of singular measures in this context and provides the necessary foundation for  inner functions.

\begin{lemma} \label{lem:boundary_lemma}
Let $\mu$ be a Borel probability measure on the unit circle $\T$, and let $F(z)$ be its Cauchy transform
\[ F(z) = \int_{\T}\frac{1}{1 - z\overline{\xi}}\,d\mu(\xi), \;\; \text{  $z \in \D$}. \]
Then, the radial boundary values $F(e^{it}) = \lim_{r\to1^-} F(re^{it})$ exist for almost every $t$. Moreover, the real part of this boundary function satisfies
\[ \Re F(e^{it}) = \frac{1}{2} \quad \text{for almost every } t \text{ with respect to the Lebesgue measure} \]
if and only if the measure $\mu$ is singular with respect to the normalized Lebesgue measure on $\T$.
\end{lemma}

\begin{proof}
The proof proceeds by relating $F(z)$ to an analytic function with a positive real part, for which the connection between its boundary values and its representing measure is well-known.

First, we have previously established that for any non-trivial probability measure $\mu$, $\Re F(z) > 1/2$ for all $z \in \D$. Let us define a new function $H(z) = 2F(z) - 1$. For $z \in \D$, its real part is $\Re H(z) = 2\Re F(z) - 1 > 2(1/2) - 1 = 0$. Thus, $H(z)$ is an analytic function with a positive real part.

We recall that a function $\widetilde{H}(z)$ is analytic with a positive real part in $\D$ if and only if it admits a unique Herglotz-Riesz integral representation with respect to a unique positive measure $\nu$ on $\T$ as
\[ \widetilde{H}(z) = \int_{\T} \frac{\xi+z}{\xi-z}\,d\nu(\xi). \]
Let us show that our measure $\mu$ is precisely this representing measure $\nu$: 
\begin{align*}
H(z) &= 2F(z) - 1 \\
&= \int_{\T} \left( \frac{2}{1-z\overline{\xi}} - 1 \right) \,d\mu(\xi) \\
&= \int_{\T} \frac{1+z\overline{\xi}}{1-z\overline{\xi}} \,d\mu(\xi).
\end{align*}
Since $\xi$ is on the unit circle, $\overline{\xi} = 1/\xi$. So, $z\overline{\xi} = z/\xi$. Substituting this gives:
\[ H(z) = \int_{\T} \frac{1+z/\xi}{1-z/\xi} \,d\mu(\xi) = \int_{\T} \frac{\xi+z}{\xi-z} \,d\mu(\xi). \]
This matches the Herglotz representation $\widetilde{H}$, confirming that the representing measure $\nu$ is our original measure $\mu$. 

A fundamental theorem of harmonic analysis (such as theorems on the boundary behavior of Poisson type integrals including Theorem 11.24 in \cite{rudin1987real}) states that the real part of the boundary values of $H(z)$ is equal to the Radon-Nikodym derivative of the absolutely continuous part of its representing measure, $\mu_{ac}$, with respect to the Lebesgue measure $m$.
Let $d\mu = d\mu_{ac} + d\mu_s = f(\xi)dm(\xi) + d\mu_s(\xi)$, where $f$ is the density. Then,
\[ \Re H(e^{it}) = \frac{d\mu_{ac}}{dm}(e^{it}) = f(e^{it}), \quad \text{for almost every } t. \]

We can now prove both directions of the lemma.
\begin{itemize}
    \item[($\implies$)] Assume $\mu$ is purely singular. Since absolutely continuous part $\mu_{ac}$ is the zero measure, so $f(\xi)=0$ a.e. From the above relation, $\Re H(e^{it}) = 0$ a.e. Since $H=2F-1$,  $\Re F(e^{it}) = 1/2$ a.e.

    \item[($\impliedby$)] If $\Re F(e^{it}) = 1/2$ a.e., then this implies $\Re H(e^{it}) = 2(1/2) - 1 = 0$ a.e., and $f(e^{it}) = \frac{d\mu_{ac}}{dm}(e^{it})$ must be zero almost everywhere. If the density of the absolutely continuous part is zero, then the absolutely continuous part itself must be the zero measure ($\mu_{ac}=0$). Therefore, the measure is composed entirely of its singular part, $\mu=\mu_s$, meaning $\mu$ is purely singular.
\end{itemize}
This completes the proof of the equivalence.
\end{proof}

\subsection{The generating function for
\texorpdfstring{$\alpha_n^{(\lambda)}$}{alpha\_n (lambda)}}
\label{sec:The Generating Function}
Consider the generating function $A^{(\lambda)}(z)$ for the coefficient $\alpha_n^{(\lambda)}$ in Definition~\ref{def:equivalence_alpha} denoted by $A^{(\lambda)}(z) = \sum_{n=0}^\infty \alpha_n^{(\lambda)} z^n$. Then the following lemma shows that $A^{(\lambda)}(z)  = \frac{\lambda}{1 - \lambda + \lambda F(z)}$:
\begin{lemma}
\label{lem:generating_function_A}
Let $A^{(\lambda)}(z) = \sum_{n=0}^\infty \alpha_n^{(\lambda)} z^n$, where $\alpha_n^{(\lambda)}$ is the coefficient in Definition~\ref{def:equivalence_alpha}. Then, 
$A^{(\lambda)}(z)  = \frac{\lambda}{1 - \lambda + \lambda F(z)}$.
\end{lemma}

\begin{proof}

We want to find the generating function \(A^{(\lambda)}(z) = \sum_{n=0}^\infty \alpha_n^{(\lambda)} z^n\), where the coefficients satisfy:
\begin{itemize}
    \item \textbf{Base Case:} \(\alpha_0^{(\lambda)} = \lambda\)
    \item \textbf{Recursion:} \(\alpha_n^{(\lambda)} = -\lambda \sum_{k=0}^{n-1} \widehat{\mu}(n-k) \alpha_k^{(\lambda)}\) for \(n \ge 1\).
\end{itemize}
Let \(F(z) = \sum_{n=0}^\infty \widehat{\mu}(n) z^n\) be the generating function for the Fourier--Stieltjes coefficients of the probability measure \(\mu\). Note that \(\widehat{\mu}(0) = 1\).

We start with the definition of \(A^{(\lambda)}(z)\) and substitute the recursion for \(n \ge 1\):
\begin{align*}
A^{(\lambda)}(z) &= \alpha_0^{(\lambda)} + \sum_{n=1}^\infty \alpha_n^{(\lambda)} z^n \\
&= \lambda - \lambda \sum_{n=1}^\infty \left( \sum_{k=0}^{n-1} \widehat{\mu}(n-k) \alpha_k^{(\lambda)} \right) z^n \label{eq:gen_func_step1_0based} \tag{1}\\
&= \lambda - \lambda \sum_{k=0}^\infty \sum_{m=1}^\infty \widehat{\mu}(m) \alpha_k^{(\lambda)} z^{m+k}\\
&= \lambda - \lambda \left( \sum_{k=0}^\infty \alpha_k^{(\lambda)} z^k \right) \left( \sum_{m=1}^\infty \widehat{\mu}(m) z^m \right)\\
&= \lambda - \lambda \Bigl[ A^{(\lambda)}(z) ( F(z) - 1 ) \Bigr],
\end{align*}
where in the last equality, we have used \(F(z) - \widehat{\mu}(0) = F(z) - 1\). 

Solving for \(A^{(\lambda)}(z)\) yields
\[ A^{(\lambda)}(z) = \frac{\lambda}{1 - \lambda + \lambda F(z)}. \]
This confirms the identity for the generating function.
\end{proof}

\subsection{The main equivalence theorem: $\lambda$-effectiveness of the Fourier exponential functions for Borel singular measure}

We begin by stating the main theorem that connects the behavior of the generalized Kaczmarz algorithm to the properties of the underlying spectral measure $\mu$.

\begin{theorem}[Generalized Mycielski–Herr–Weber]
\label{thm:Generalized Mycielski–Herr–Weber}
Let $\{e_n\}_{n\ge0}$ be the Fourier exponential functions $e_n(x) = \{ e^{2\pi i n x}\}_{n=0}^\infty$ on $\T$. 
Let
\[
F(z)=\int_{\T}\frac{1}{1 - z\overline\xi}\,d\mu(\xi)
\]
be its Cauchy transform associated with the spectral measure~$\mu$ on~$\T$, and fix $\lambda\in(0,2)$ with $\lambda \neq 1$. Consider the generating function $A^{(\lambda)}(z) = \sum_{n=0}^\infty \alpha_n^{(\lambda)}\,z^n$ associated with $F(z)$ from Section~\ref{sec:The Generating Function}.
Then the following statements are equivalent:
\begin{enumerate}
  \item \textbf{\(\lambda\)-effectiveness:} For every $f\in\Hcal$, the relaxed Kaczmarz iterates $w_n = w_{n-1} + \lambda\langle f - w_{n-1},\varphi_n\rangle\varphi_n$, with $w_{-1}=0$, converge in norm to~$f$.

  \item The coefficients of the generating function satisfy
    \[
      \sum_{n=1}^\infty\bigl|\alpha_n^{(\lambda)}\bigr|^2 = \frac{\lambda^3}{2-\lambda}.
    \]

  \item   The integral of the squared modulus on the circle at $\lambda$ is
    \[
      \int_{\T}\bigl|A^{(\lambda)}(e^{it})-\lambda\bigr|^2\,\frac{dt}{2\pi} = \frac{\lambda^3}{2-\lambda}.
    \]
  \item  \textbf{Spectral–measure singularity for $\lambda \neq 1$:} The spectral measure $\mu$ is singular with respect to the Lebesgue measure. 
\end{enumerate}
\end{theorem}

\begin{remark}
In the case when $\lambda = 1$ (the regular Kaczmarz algorithm case), the Fourier exponential functions
are effective if and only if either the spectral measure $\mu$ is identical to the normalized Lebesgue measure or is singular with respect to the Lebesgue measure, which is shown by Mycielski et al. in \cite{kwapien2001kaczmarz}.
\end{remark}

\subsection{Proof of the Theorem~\ref{thm:Generalized Mycielski–Herr–Weber}}

We prove the equivalences in the theorem in the following order: (i) $\iff$ (ii), (ii) $\iff$ (iii), and (iii) $\iff$ (iv). 

The proof starts with making use of the operators $\widetilde{T}_n$ and $\widetilde{Q}_n$  as established in  Section 4. 

The error of the Kaczmarz iteration $w_n = w_{n-1} + \lambda P_n(f-w_{n-1})$ (if started at $n=0$ with $w_{-1}=0$) would satisfy $f - w_n = \widetilde{T}_n f$. By the definition of $\lambda$-effectiveness, this implies that $\widetilde{T}_n f \to 0$.
From the preliminaries in Section 4, we have the identity $I - \widetilde{T}_n = \sum_{j=0}^n \widetilde{Q}_j$ and Lemma~\ref{lem:Generalized_Kaczmarz_Scaled_Parseval} implies that $f = \sum_{j=0}^\infty \widetilde{Q}_j f$. Recall that the operators $\widetilde{Q}_n$ have the representation $\widetilde{Q}_n f = e_n \langle h_n^{(\lambda)}, f \rangle_{L^2(\mu)}$ for $n \ge 0$. Hence, $\{e_n\}_{n\ge0}$ is $\lambda$-effective if and only if $\sum_{j=0}^n \widetilde{Q}_j f = \sum_{j=0}^n e_j \langle h_j^{(\lambda)}, f \rangle_{L^2(\mu)} \rightarrow f$ as $n \rightarrow \infty$. 

Let $K_n(f) := \sum_{j=0}^n e_j \langle h_j^{(\lambda)}, f \rangle_{L^2(\mu)}$. Then, Lemma~\ref{lem:h_alpha_relation_0based} implies that 
\begin{align}
K_n(f) = \sum_{k=0}^n \left(\sum_{s=0}^k \alpha_{k-s}^{(\lambda)} \langle e_s, f \rangle_{L^2(\mu)}\right)e_k,    
\end{align}
and since $\{e_i\}_{i\ge0}$ is linearly dense by Lemma~\ref{lem:singular_Borel_measure_exponential_function}, this statement is equivalent to saying that $\{e_i\}_{i\ge0}$ is $\lambda$-effective if and only if $\lim_{n \rightarrow \infty} K_n(e_i) = e_i$, for $i=0,1,2, \dots$.

\subsection{Proof of (i) $\iff$ (ii): effectiveness and the coefficient sum}

Let
\[
u_n:=\sum_{l=0}^{n}\alpha_l^{(\lambda)}\,e_l
\quad \text{and} \quad
v_{n,i}:=\sum_{l=0}^{\,n-i}\alpha_l^{(\lambda)}\,e_{\,l+i}\quad(n\ge i).
\]
Using our stationarity convention
\(\langle e_p,e_q\rangle=\widehat{\mu}(p-q)\), we have
\begin{align*}
\|v_{n,i}\|^2
&=\Big\langle \sum_{l=0}^{n-i}\alpha_l^{(\lambda)}e_{l+i},\;
            \sum_{m=0}^{n-i}\alpha_m^{(\lambda)}e_{m+i}\Big\rangle\\
&=\sum_{l,m=0}^{n-i}\overline{\alpha_l^{(\lambda)}}\,\alpha_m^{(\lambda)}
   \,\langle e_{l+i},e_{m+i}\rangle\\
&=\sum_{l,m=0}^{n-i}\overline{\alpha_l^{(\lambda)}}\,\alpha_m^{(\lambda)}
   \,\langle e_{l},e_{m}\rangle\\
&=\Big\langle \sum_{l=0}^{n-i}\alpha_l^{(\lambda)}e_l,\;
            \sum_{m=0}^{n-i}\alpha_m^{(\lambda)}e_m\Big\rangle
=\|S_{n-i}\|^2,
\end{align*}
where the third equality  is due to the stationary property.
Therefore,
\begin{equation}\label{eq:norm-shift}
\|v_{n,i}\|^2=\|u_{n-i}\|^2, \quad n\ge i.
\end{equation}

The following identities are proved in Appendix B:
\begin{equation}\label{eq:shift-reduction-fixed}
K_n(e_i)-e_i
=\frac{\lambda-1}{\lambda}\,v_{n,i}
+\sum_{j=1}^{i}\widehat{\mu}(-j)\,v_{n,\,i-j}
\quad(i\ge1,\;n\ge i),
\end{equation}
and for $i=0$,
\begin{equation}\label{eq:i0-identity-fixed}
K_n(e_0)-e_0=\frac{\lambda-1}{\lambda}\,u_n\qquad(n\ge0).
\end{equation}
One can also easily check that for $\lambda = 1$, the above Equations \eqref{eq:shift-reduction-fixed} and \eqref{eq:i0-identity-fixed} reduce to the equations derived in the proof of Theorem 2 in \cite{kwapien2001kaczmarz}, indicating that \eqref{eq:shift-reduction-fixed} and \eqref{eq:i0-identity-fixed} generalize those in \cite{kwapien2001kaczmarz}.

We see by \eqref{eq:norm-shift} that each term on the right-hand side of
\eqref{eq:shift-reduction-fixed} has norm controlled by some $\|u_{\cdot}\|$.
Since $\lambda \in (0,2)$ and not equal to $1$, we have
\[
\{e_k\}\text{ is $\lambda$-effective}
\quad\Longleftrightarrow\quad
\|u_n\|\longrightarrow 0.
\]

Set $r_n^2:=\|u_n\|^2$. Since $u_n=u_{n-1}+\alpha_n^{(\lambda)}e_n$,
\begin{align*}
r_n^2
&=\|u_{n-1}\|^2+\|\alpha_n^{(\lambda)}e_n\|^2
  +2\Re\!\big\langle u_{n-1},\,\alpha_n^{(\lambda)}e_n\big\rangle \\
&=r_{n-1}^2+|\alpha_n^{(\lambda)}|^2
  +2\Re\!\left(\alpha_n^{(\lambda)}
      \sum_{k=0}^{n-1}\overline{\alpha_k^{(\lambda)}}\,\widehat{\mu}(k-n)\right).
\end{align*}
Taking conjugates of the recurrence
\(
\alpha_n^{(\lambda)}=-\lambda\sum_{t=1}^{n}\widehat\mu(t)\,\alpha_{n-t}^{(\lambda)}
\)
and using $\widehat{\mu}(k-n)=\overline{\widehat{\mu}(n-k)}$ gives
\[
\sum_{k=0}^{n-1}\overline{\alpha_k^{(\lambda)}}\,\widehat{\mu}(k-n)
=-\frac{1}{\lambda}\,\overline{\alpha_n^{(\lambda)}},
\]
hence
\[
r_n^2
=r_{n-1}^2+\Bigl(1-\frac{2}{\lambda}\Bigr)\,|\alpha_n^{(\lambda)}|^2,
\qquad
r_0^2=\|\alpha_0^{(\lambda)}e_0\|^2=\lambda^2.
\]
Iterating, we get
\[
r_n^2=\lambda^2+\Bigl(1-\frac{2}{\lambda}\Bigr)\sum_{k=1}^{n}|\alpha_k^{(\lambda)}|^2,
\]
and therefore, $\|u_n\|\to0$ iff
\[
0=\lambda^2+\Bigl(1-\frac{2}{\lambda}\Bigr)\sum_{k=1}^{\infty}|\alpha_k^{(\lambda)}|^2
\quad\Longleftrightarrow\quad
\sum_{k=1}^{\infty}|\alpha_k^{(\lambda)}|^2=\frac{\lambda^3}{2-\lambda}.
\]
This proves \textup{(i)} $\iff$ \textup{(ii)}.

\subsection{Proof of (ii) $\iff$ (iii): the coefficient sum and the integral}
This equivalence is a direct consequence of Parseval's theorem. As we have shown in Section 7.2, since $\Re F(z) >  \frac{1}{2}$ and the denominator of $A^{(\lambda)}(z)$ is $1- \lambda + \lambda F(z)$, the function $A^{(\lambda)}(z) - \lambda$ is analytic in the unit disk $\D$ and belongs to the Hardy space $H^2(\D)$. Since its Taylor series at $z=0$ is $\sum_{n=1}^\infty \alpha_n^{(\lambda)} z^n$, by Parseval's theorem, we have 
\[
\sum_{n=1}^\infty |\alpha_n^{(\lambda)}|^2 = \int_{\T} |A^{(\lambda)}(e^{it}) - \lambda|^2 \,\frac{dt}{2\pi}.
\]
This shows that the identity in (ii) is equivalent to the integral form in (iii).

\subsection{Proof of (iii) $\iff$ (iv): The integral identity and singularity}
This is the core of the proof. We show that the integral identity holds if and only if the measure $\mu$ is singular. Let $g(z) = \frac{A^{(\lambda)}(z)-\lambda}{\lambda} = \frac{\lambda(1-F(z))}{1-\lambda+\lambda F(z)}$, where we have used Lemma~\ref{lem:generating_function_A}.  The integral identity (iii) is equivalent to proving $\|g\|_{H^2}^2 = \int_{\T}|g(e^{it})|^2 \frac{dt}{2\pi} = \frac{\lambda}{2-\lambda}$.

\begin{proof}[Proof of (iv) $\implies$ (iii)]
Assume $\mu$ is a singular measure. A key property of the Cauchy transform is that this implies its boundary values satisfy $\Re F(e^{it}) = \frac{1}{2}$ for almost every $t$ as established in Lemma~\ref{lem:boundary_lemma}.

Now, consider the linear fractional transformation $h(w) = \frac{\lambda(1-w)}{1-\lambda+\lambda w}$. This transformation maps the line $\Re w = 1/2$ to a circle. Let's find the center and the radius.
The points $w_1 = 1/2$ and $w_2 \to \infty$ on the line are mapped to:
\begin{align*}
h(1/2) &= \frac{\lambda(1-1/2)}{1-\lambda+\lambda/2} = \frac{\lambda}{2-\lambda} \\
h(\infty) &= -1.
\end{align*}
These two points form the diameter of the image circle. The center $c$ and radius $r$ are
\[
c = \frac{1}{2}\left(\frac{\lambda}{2-\lambda} - 1\right) = \frac{\lambda-1}{2-\lambda}, \quad r = \frac{1}{2}\left(\frac{\lambda}{2-\lambda} - (-1)\right) = \frac{1}{2-\lambda}.
\]
The function $g(z)=h(F(z))$ is analytic in $\D$. Since $\Re F(z) > 1/2$ in $\D$, $g(z)$ maps $\D$ into the interior of this image circle, i.e., $|g(z)-c|<r$.
Since $\mu$ is singular, for a.e. $t$, $\Re F(e^{it})=1/2$, which means the boundary values $g(e^{it})$ must lie on the image circle. 
\[
|g(e^{it}) - c|^2 = r^2 \quad \text{for a.e. $t$}.
\]
Expanding this identity gives $|g(e^{it})|^2 - 2\Re(\bar{c} g(e^{it})) + |c|^2 = r^2$.
Then, 
\[
\int_{\T} |g(e^{it})|^2\,\frac{dt}{2\pi} - 2\Re\left(\bar{c} \int_{\T} g(e^{it})\,\frac{dt}{2\pi}\right) + |c|^2 = r^2.
\]
For any $H^2$ function $g$, $\int_{\T} g(e^{it})\,\frac{dt}{2\pi} = g(0)$. We have $F(0)=1$, so $g(0) = h(F(0)) = h(1) = 0$.
Thus, so the identity simplifies to
\[
\int_{\T} |g(e^{it})|^2\,\frac{dt}{2\pi} + |c|^2 = r^2.
\]
Therefore,  we have
\[
\|g\|_{H^2}^2 = r^2 - |c|^2 = \left(\frac{1}{2-\lambda}\right)^2 - \left(\frac{\lambda-1}{2-\lambda}\right)^2
= \frac{\lambda}{2-\lambda}.
\]
This is precisely the condition equivalent to (iii). Thus, if $\mu$ is singular, the integral identity holds.
\end{proof}

\begin{proof}[Proof of (iii) $\implies$ (iv)]
Assume that the integral identity holds, which means $\|g\|_{H^2}^2 = \frac{\lambda}{2-\lambda}$. From the derivation above, this is equal to $r^2 - |c|^2$. The boundary values $g(e^{it})$ exist a.e., and since $g$ maps the disk $\D$ into the disk $|w-c|<r$,  $|g(e^{it})-c| \le r$ a.e.

The calculation $\|g\|_{H^2}^2 = \int |g|^2 = r^2 - |c|^2$ relied on the fact that $\int \Re(\bar{c}g) = 0$, and hence,
\[
r^2 - |c|^2 - \int |g|^2 = 0.
\]
Using $|g|^2 = |(g-c)+c|^2  = |g-c|^2+2\Re(\bar{c}g) - 2|c|^2 + |c|^2$, and  $\int g=g(0)=0$, this becomes
\[
r^2 - \int |g(e^{it})-c|^2\,\frac{dt}{2\pi} = 0.
\]
We have an integral of a non-negative function $r^2 - |g(e^{it})-c|^2$ (since $|g-c|\le r$ a.e.), and the integral is zero. This implies the integrand must be zero almost everywhere.
\[
|g(e^{it}) - c|^2 = r^2 \quad \text{for a.e. } t.
\]
This means the boundary values of $g(z)$ lie on the circle $|w-c|=r$. Since $h$ is a Mobius transform, it maps lines/circles to lines/circles. For the image $g(e^{it})=h(F(e^{it}))$ to be on the boundary of the image disk, the pre-image $F(e^{it})$ must lie on the boundary of the pre-image domain, which is the line $\Re w=1/2$.
Therefore, Lemma~\ref{lem:boundary_lemma}, $\Re F(e^{it}) = 1/2$ a.e. implies that the measure $\mu$ is singular with respect to the Lebesgue measure. This shows (iii) $\implies$ (iv).
\end{proof}

\subsection{Application: $\lambda$-effectiveness of the Fourier exponential functions}

\begin{theorem}[$\lambda$-effectiveness of the Fourier exponential functions]
\label{thm:lambda_effectiveness_exponential_functions}
Let $\mu$ be a singular Borel probability measure on $[0,1)$. Then the system of Fourier exponential functions with nonnegative indices, \(\{e^{2\pi i n x}\}_{n\ge0}\)
\(
\)
is $\lambda$-effective in $L^2(\mu)$ for every $\lambda\in(0,2)$.
\end{theorem}

\begin{proof}
The case when $\lambda = 1$ follows from Theorem 2  in \cite{kwapien2001kaczmarz}.
We apply the criterion provided by Theorem~\ref{thm:Generalized Mycielski–Herr–Weber} for $\lambda \neq 1$, and we verify that the measure $\mu$ is the spectral measure for the exponential system $\{e_n\}_{n\ge 0}$ as follows.

\emph{Stationarity:} As established in Section~7, the system $\{e_n\}_{n\ge 0}$ is stationary, and its spectral measure is $\mu$. \emph{Linear density:} Since $\mu$ is singular, Lemma~\ref{lem:singular_Borel_measure_exponential_function} guarantees that $\{e_n\}_{n\ge 0}$ is linearly dense in $L^2(\mu)$. Therefore, the hypotheses of Theorem~\ref{thm:Generalized Mycielski–Herr–Weber} are satisfied and because the spectral measure is singular (Condition~(iv)), it follows that the exponential system is $\lambda$-effective (Condition~(i)) for every $\lambda\in(0,2)$.
\end{proof}

\subsection{Difference between \texorpdfstring{$\lambda=1$}{lambda=1} case and the other cases}

\subsection*{Inner function for the \texorpdfstring{$\lambda=1$}{lambda=1} case}
Note that when $\lambda = 1$, $g(z) = \frac{A^{(\lambda)}(z)-\lambda}{\lambda} = \frac{\lambda(1-F(z))}{1-\lambda+\lambda F(z)}$ in the proof of Theorem~\ref{thm:Generalized Mycielski–Herr–Weber} (in the step showing the equivalence between $(iii)$ and $(iv)$) reduces to
$ \frac{1 - F(z)}{F(z)}$.

For simplicity, define
\[
\phi_1(z)\;=\;\frac{1 - F(z)}{F(z)}.
\]
We now show this function is inner by checking its properties inside the unit disk and on its boundary.

\paragraph{In the disk \(\lvert z\rvert<1\):}
Since $\Re F(z)>\tfrac12$ for any non-trivial probability measure $\mu$, we have:
\[
\lvert\phi_1(z)\rvert
=\frac{\lvert 1 - F(z)\rvert}{\lvert F(z)\rvert}
< 1,
\]
because the condition $\Re w > 1/2$ implies that the point $w$ is closer to 1 than to 0. Thus, $\phi_1$ is a bounded analytic self‐map of the unit disk $\mathbb{D}$.

\paragraph{On the circle \(z=e^{it}\):}
Using the established fact that $\Re F(e^{it})=\tfrac12$ for our singular measure $\mu$, and writing $F(e^{it})=\tfrac12 + i\,v(t)$, we get:
\[
\phi_1(e^{it})
=\frac{1 - \bigl(\tfrac12 + i\,v(t)\bigr)}{\tfrac12 + i\,v(t)}
=\frac{\tfrac12 - i\,v(t)}{\tfrac12 + i\,v(t)}.
\]
The modulus is therefore:
\[
\bigl\lvert\phi_1(e^{it})\bigr\rvert
=\frac{\sqrt{\tfrac14 + v(t)^2}}{\sqrt{\tfrac14 + v(t)^2}}
=1
\quad\text{for a.e. }t.
\]

Since $\phi_1(z)$ is analytic on $\mathbb{D}$, satisfies $|\phi_1(z)|<1$ for $|z|<1$, and its radial boundary values satisfy $|\phi_1(e^{it})|=1$ a.e., it follows by definition that $\phi_1(z)$ is an inner function.

\paragraph{Implications of \(\phi_1\) being Inner:}

Since
\[
\phi_1(z)=\frac{1 - F(z)}{F(z)},
\]
we have shown
\[
|\phi_1(z)|<1\;\;(z\in\mathbb{D}),
\qquad
|\phi_1(e^{it})|=1\;\text{a.e.},
\]
so $\phi_1$ is an inner function.

\begin{remark}[Inner function for the general $\lambda$ case]
For the general $\lambda$ case, however, simple scaling of $g(z)$ does not produce an inner function; if we set $\sqrt{\frac{2-\lambda}{\lambda^3}} g(z) = \sqrt{\frac{2-\lambda}{\lambda^3}} \tfrac{\lambda - \lambda F(z)}{1 - \lambda + \lambda F(z)}$ based on $(ii)$ of Theorem~\ref{thm:Generalized Mycielski–Herr–Weber} is not an inner function. But it turns out that an affine transformation of $g$ produces the inner function as it is implied in Steps $(iii)$ and $(iv)$ of the proof of Theorem~\ref{thm:Generalized Mycielski–Herr–Weber}. To clarify the affine transformation, we extract the part of the proof of the inner function in the next subsection. 
\end{remark}

\subsection{The inner function for \texorpdfstring{$\lambda$}{lambda}-effectiveness}

In the proof of Theorem~\ref{thm:Generalized Mycielski–Herr–Weber} concerning Kaczmarz $\lambda$- effectiveness, we have introduced the auxiliary function $g(z) = h(F(z))$, where $h(w) = \frac{\lambda(1-w)}{1-\lambda+\lambda w}$. We showed that $h$ maps the half-plane $\Re w > 1/2$ to the open disk $|w-c|<r$, where
\[ c = \frac{\lambda-1}{2-\lambda} \quad \text{and} \quad r = \frac{1}{2-\lambda}. \]
The following proposition formalizes the idea of using the normalized function $\phi(z) = \frac{g(z)-c}{r}$. 

\begin{proposition}
Let $F(z)$ be the Cauchy transform of a probability measure $\mu$ on $\T$, and let $\lambda \in (0,2)$. Define the function $\phi(z)$ by
\[ \phi(z) = \frac{g(z)-c}{r}, \quad \text{where} \quad g(z) = \frac{\lambda(1-F(z))}{1-\lambda+\lambda F(z)}. \]
This function has the explicit form
\[ \phi(z) = \frac{1 - \lambda F(z)}{1-\lambda+\lambda F(z)}. \]
The condition that $\mu$ is a singular measure (and thus that the Kaczmarz algorithm is effective) is equivalent to the condition that this $\phi(z)$ is an inner function.
\end{proposition}

\begin{proof}
The proof consists of two parts: first, we derive the explicit form of $\phi(z)$, and second, we prove that it is an inner function if and only if $\mu$ is singular.

\textbf{1. Derivation of the explicit form.}
We substitute the expressions for $g(z)$, $c$, and $r$ into the definition of $\phi(z)$.
\begin{align*}
\phi(z) &= \frac{1}{r} \left( g(z) - c \right) \\
&= (2-\lambda) \left( \frac{\lambda(1-F(z))}{1-\lambda+\lambda F(z)} - \frac{\lambda-1}{2-\lambda} \right) \\
&= \frac{(2-\lambda)\lambda(1-F(z))}{1-\lambda+\lambda F(z)} - (\lambda-1)\\
&= \frac{(2\lambda-\lambda^2)(1-F(z)) - (\lambda-1)(1-\lambda+\lambda F(z))}{1-\lambda+\lambda F(z)} \\
&=  \frac{1 - \lambda F(z)}{1-\lambda+\lambda F(z)}.
\end{align*}
This confirms the simple explicit form, also confirming that when $\lambda=1$, it reduces to the negative of the known inner function $1 - 1/F(z)$ in \cite{kwapien2001kaczmarz, kwapien2006erratum}.

\textbf{2. Proof of equivalence.}
We now show that $\phi(z)$ is an inner function if and only if $\mu$ is singular. An inner function must be analytic in $\D$, satisfy $|\phi(z)|\le1$ in $\D$, and have boundary values $|\phi(e^{it})|=1$ a.e. on $\T$.

\textit{(Analyticity):} $F(z)$ is analytic in $\D$. The denominator $1-\lambda+\lambda F(z)$ is never zero in $\D$ because $\Re F(z)>1/2$ and the pole of the transformation $w \mapsto 1/(1-\lambda+\lambda w)$ is at $w= (\lambda-1)/\lambda$, which has real part less than $1/2$. Thus, $\phi(z)$ is analytic.

\textit{(Boundedness):} As shown in the main theorem's proof, $g(z)$ maps the unit disk into the open disk $|w-c|<r$. This means $|g(z)-c|<r$ for all $z\in\D$. By definition of $\phi(z)$, this is equivalent to $|\phi(z)| = \frac{|g(z)-c|}{r} < 1$. So the boundedness condition is always satisfied for any probability measure $\mu$.

\textit{(Boundary values):} The crucial condition is $|\phi(e^{it})|=1$ a.e. This is equivalent to $|g(e^{it})-c|=r$ a.e., which means the boundary values of $g(z)$ must lie on the boundary of the disk $D_c$. As shown in the proof of the main theorem, this happens if and only if the boundary values of $F(z)$ lie on the line $\Re w = 1/2$.
And as established by the lemma we proved previously, the condition $\Re F(e^{it})=1/2$ a.e. holds if and only if the measure $\mu$ is purely singular.

Therefore, $\phi(z)$ satisfies all three conditions for being an inner function if and only if $\mu$ is a singular measure. This completes the proof.
\end{proof}

\section{Expansion based on the generalized Kaczmarz algorithm and \(\lambda\)-dependent general form of the normalized Cauchy transform}

\subsection{Series expansion based on the generalized Kaczmarz algorithm}
The following result provides a series expansion in terms of the Fourier exponential functions $\{e_n\}_{n\ge 0}$ in $L^2(\mu)$ based on the generalized Kaczmarz algorithm, which also ties to the first half of this paper about its regret bounds. To compare our result with the previous literature, the sequence $\{e_n\}$ starts from $n=0$, not $n=1$.

\begin{theorem}[Generalized Kaczmarz Fourier series representation with step size $\lambda$]
Let $\mu$ be a singular Borel probability measure on $[0,1)$ and let $e_n(x) := \{e^{2\pi i n x}\}_{n\ge 0}$ be the Fourier exponential functions in $L^2(\mu)$. Let $P_n = e_n e_n^*$ be the rank-one projection onto the span of $e_n$ for $n \ge 0$. Fix a relaxation parameter $\lambda\in (0,2)$. Define a sequence $\{h_n^{(\lambda)}\}_{n\ge 0}$ recursively by
\begin{equation} \label{eq:h_recursion_0based}
h_0^{(\lambda)} = \lambda\,e_0,\qquad
h_n^{(\lambda)} = \lambda\Bigl(e_n-\sum_{k=0}^{n-1}\overline{\langle e_n,e_k\rangle_{L^2(\mu)}}\,h_k^{(\lambda)}\Bigr),\quad n\ge 1.
\end{equation}
Let $\widetilde{Q}_n = \lambda P_n \widetilde{T}_{n-1}$ for $n \ge 0$. Then,
\begin{enumerate}

    \item Every $f\in L^2(\mu)$ has a convergent series expansion
    \[
    f = \sum_{n=0}^\infty \widetilde{Q}_n f = \sum_{n=0}^\infty e_n \langle h_n^{(\lambda)}, f \rangle_{L^2(\mu)}.
    \]
    
    \item A rescaled Parseval-type identity holds:
    \[
    \sum_{n=0}^\infty \bigl\| \widetilde{Q}_n f \bigr\|_{L^2(\mu)}^2 = \sum_{n=0}^\infty \bigl|\langle h_n^{(\lambda)}, f \rangle_{L^2(\mu)}\bigr|^2 = \frac{\lambda}{2-\lambda} \|f\|^2_{L^2(\mu)}.
    \]
\end{enumerate}
\end{theorem}

\begin{proof}[Proof]
By Theorem~\ref{thm:lambda_effectiveness_exponential_functions}, we know that the Fourier exponential functions $e_n(x) := \{e^{2\pi i n x}\}_{n\ge 0}$ are $\lambda$-effective for any Borel singular measure.
Thus, Corollary~\ref{cor:rank1_1based} implies (i) and Corollary~\ref{cor:scalar_parseval_identity} implies (ii). 
\end{proof}

\subsection{Implications for Fourier coefficients and series}

Recall that in the $\lambda=1$ case (as in \cite{herr2017fourier}), the Fourier series could be expressed using coefficients defined via the combinatorial coefficients $\alpha_n^{(1)}$. We extend this to the relaxed setting.

In the relaxed setting with $\lambda\in (0,2)$, we have the sequence $\{h_n^{(\lambda)}\}_{n\ge 0}$ related to the original basis $\{e_n\}_{n\ge 0}$ via the coefficients $\alpha_n^{(\lambda)}$ (defined combinatorially for $n \ge 0$) by Lemma \ref{lem:h_alpha_relation_0based}:
\[
h_n^{(\lambda)} = \sum_{j=0}^{n}\overline{\alpha_{n-j}^{(\lambda)}}\,e_j, \quad \text{for } n \ge 0.
\]
The generalized Kaczmarz framework provides the series representation for $f \in L^2(\mu)$ as
\[
f = \sum_{n=0}^\infty e_n \langle h_n^{(\lambda)}, f \rangle_{L^2(\mu)}.
\]
Let $d_n^{(\lambda)} = \langle h_n^{(\lambda)}, f \rangle_{L^2(\mu)}$ be the coefficient multiplying $e_n$ for $n \ge 0$. Using Lemma~\ref{lem:h_alpha_relation_0based}, we can express $d_n^{(\lambda)}$ in terms of the standard Fourier coefficients $\widehat{f}(j) = \langle e_j, f\rangle_{L^2(\mu)}$ for $j \ge 0$ as
\begin{align*}
d_n^{(\lambda)} &= \langle h_n^{(\lambda)}, f \rangle_{L^2(\mu)} \quad (\text{for } n \ge 0) \\
&= \left\langle \sum_{j=0}^{n}\overline{\alpha_{n-j}^{(\lambda)}}\,e_j , f \right\rangle_{L^2(\mu)} \\
&= \sum_{j=0}^{n} \overline{\overline{\alpha_{n-j}^{(\lambda)}}} \langle e_j, f \rangle_{L^2(\mu)} \\
&= \sum_{j=0}^{n} \alpha_{n-j}^{(\lambda)} \, \widehat{f}(j).
\end{align*}
Note that since our inner product $\inner{\cdot, \cdot}$ is linear in the second argument and conjugate linear in the first argument, $ \langle e_j, f \rangle_{L^2(\mu)} = \widehat{f}(j)$. 
Thus, the Fourier series expansion can be written explicitly in terms of the standard coefficients $\widehat{f}(j)$ and the combinatorial coefficients $\alpha_k^{(\lambda)}$ as:
\[
f(x) = \sum_{n=0}^\infty d_n^{(\lambda)} e_n(x) = \sum_{n=0}^\infty \left( \sum_{j=0}^{n} \alpha_{n-j}^{(\lambda)} \, \widehat{f}(j) \right) e_n(x),
\]
where $e_n(x) = e^{2\pi i n x}$. This shows how the generalized Kaczmarz procedure implicitly re-weights and combines the standard Fourier coefficients (conjugated) to form the coefficients for the reconstruction using the original basis $\{e_n\}_{n\ge 0}$.

\subsection{\texorpdfstring{$\lambda$}{lambda}-Dependent general form of the normalized Cauchy transform}

Let \(N(z) = \sum_{n=0}^\infty \widehat{f}(n) z^n\) and \(F(z) = \sum_{n=0}^\infty \widehat{\mu}(n) z^n\) (where $\widehat{f}(n) = \langle e_n,f \rangle$, $e_n(x)=e^{2\pi i n x}$). Recall that \(A^{(\lambda)}(z) = \sum_{n=0}^\infty \alpha_n^{(\lambda)} z^n = \frac{\lambda}{1 - \lambda + \lambda F(z)}\) as defined in Section~\ref{sec:The Generating Function}. 

The standard normalized Cauchy transform $V_\mu f(z) = N(z)/F(z)$, where $N(z) = \sum_{n=0}^\infty \widehat{f}(n) z^n$ and $F(z) = \sum_{n=0}^\infty \widehat{\mu}(n) z^n$, acts as a generating function for the regular Kaczmarz (\(\lambda=1\)) coefficients: $V_\mu f(z) = \sum_{n=0}^\infty \langle h_n^{(1)},f \rangle z^n$. This relies on $A^{(1)}(z) = \sum_{n=0}^\infty \alpha_n^{(1)} z^n = 1/F(z)$ and the relations $h_n^{(1)} = \sum_{j=0}^n \overline{\alpha_{n-j}^{(1)}} e_j$ and $K_n^{(1)} = [z^n]V_\mu f = \sum_{j=0}^n \alpha_{n-j}^{(1)} \widehat{f}(j) = \langle h_n^{(1)}, f \rangle$.

For the generalized algorithm with $\lambda \in (0, 2)$, we use the sequence $\{h_n^{(\lambda)}\}_{n\ge 0}$ defined by  \eqref{eq:h_recursion_0based}. We also have the combinatorial coefficients $\{\alpha_n^{(\lambda)}\}_{n\ge 0}$ defined without conjugates. We established the relationship $h_n^{(\lambda)} = \sum_{j=0}^{n}\overline{\alpha_{n-j}^{(\lambda)}}\,e_j$ for $n \ge 0$ (Lemma \ref{lem:h_alpha_relation_0based}). Since $A^{(\lambda)}(z) \neq 1/F(z)$ for $\lambda \neq 1$, the analytic representation connected to the Kaczmarz coefficients $\{h_n^{(\lambda)}\}_{n\ge 0}$ will differ from the standard normalized Cauchy transform $V_\mu f(z)$.

Let us define a transform using $A^{(\lambda)}(z)$ and relate its coefficients to the Kaczmarz coefficients $\langle h_n^{(\lambda)}, f \rangle$.

For \(\lambda \in (0, 2)\), \(f \in L^2(\mu)\), and \(z \in \mathbb{D}\), let \(N_f(z)\) denote the unnormalized Cauchy transform of a function \(f(x)\) with respect to \(\mu\):
\[
  N_f(z) := \int_0^1 \frac{f(x)}{1 - z\,e^{-2\pi i x}}\,d\mu(x).
\]
The \emph{normalized \(\lambda\)-transform} is defined as:
\[
  V_\mu^{(\lambda)}f(z) := N_f(z) \, A^{(\lambda)}(z) = \frac{\lambda\, N_f(z)}{1-\lambda+\lambda\,F(z)}.
\]
Or equivalently, we have 
\[
  V_\mu^{(\lambda)}f(z)
  = \frac{\lambda\,\displaystyle\int_0^1\frac{f(x)}{1 - z\,e^{-2\pi i x}}\,d\mu(x)}
         {\,1-\lambda+\lambda\displaystyle\int_0^1\frac{1}{1 - z\,e^{-2\pi i x}}\,d\mu(x)\,}.
\]
Note that when $\lambda = 1$ this expression reduces to the normalized  Cauchy transform \cite{jorgensen2020kaczmarz, herr2017fourier, cima2006cauchy}.

\begin{theorem}[Relation between \(\lambda\)-transform and Kaczmarz coefficients]
\label{thm:lambda_transform_kacz_coeffs}
Let \(\lambda \in (0, 2)\). Let $\{h_n^{(\lambda)}\}_{n\ge 0}$ be the sequence from the generalized Kaczmarz algorithm defined by the 0-based recursion \eqref{eq:h_recursion_0based}. The normalized \(\lambda\)-transform \(V_\mu^{(\lambda)} f(z)\) defined above admits the power series expansion
\[
V_\mu^{(\lambda)} f(z) = \sum_{n=0}^\infty K_n^{(\lambda)} z^n, \quad \text{for } z \in \mathbb{D},
\]
where the coefficients are given by 
\[
K_n^{(\lambda)} = \sum_{j=0}^{n}\alpha_{n-j}^{(\lambda)}\,\widehat{f}(j).
\]
These coefficients \(K_n^{(\lambda)}\) are related to the generalized Kaczmarz Fourier coefficients $d_n^{(\lambda)} = \langle h_n^{(\lambda)}, f \rangle$, $n \ge 0$ by
\[
d_n^{(\lambda)} = \sum_{j=0}^{n} \alpha_{n-j}^{(\lambda)} \, \widehat{f}(j).
\]
Therefore, $K_n^{(\lambda)}$ and $d_n^{(\lambda)}$ are the same. 

\end{theorem}

\begin{proof}
Since \(V_\mu^{(\lambda)} f(z) = N(z) A^{(\lambda)}(z)\),
\[
V_\mu^{(\lambda)} f(z) = \left( \sum_{j=0}^\infty \widehat{f}(j) z^j \right) \left( \sum_{k=0}^\infty \alpha_k^{(\lambda)} z^k \right).
\]
The coefficient of \(z^n\) in the product is precisely
\[
K_n^{(\lambda)} = \sum_{j=0}^{n} \alpha_{n-j}^{(\lambda)} \widehat{f}(j).
\]
Therefore, \(V_\mu^{(\lambda)} f(z) = \sum_{n=0}^\infty K_n^{(\lambda)} z^n\).

From Theorem~\ref{thm:lambda_transform_kacz_coeffs}, we now that the actual Kaczmarz coefficients $d_n^{(\lambda)} = \langle h_n^{(\lambda)}, f \rangle$ for $n \ge 0$ is given as
\[
d_n^{(\lambda)} = \sum_{j=0}^{n} \alpha_{n-j}^{(\lambda)} \, \widehat{f}(j).
\]
\end{proof}

\section{Conclusion and Future Directions}

In this work, we have studied a comprehensive theoretical framework for analyzing the regret of block Kaczmarz algorithms in  infinite-dimensional Hilbert spaces. We derived sharp, dimension-free $O(1/k)$ average regret bounds for the generalized Kaczmarz algorithm with any relaxation parameter $\lambda \in (0,2)$. These bounds were shown to be robust, gracefully accommodating stochastic noise by separating the decaying learning term from an irreducible noise floor. The interpretation of the total cumulative regret as a necessary "cost of learning," which is fully expended by effective systems, provides a new perspective on the performance of these online algorithms.

Another main contribution of our work is the proof that the canonical system of Fourier exponential functions, $\{e^{2\pi i n x}\}_{n\ge0}$, is $\lambda$-effective for all $\lambda \in (0,2)$ if and only if the Borel probability measure is singular unless $\lambda = 1$. This result is crucial, as it validates our regret analysis for a broad and important class of non-orthogonal systems. By leveraging Hardy space theory and identifying a novel $\lambda$-dependent inner function, we generalized the existing effectiveness results for $\lambda=1$, confirming that the excellent convergence properties of Fourier-type Kaczmarz updates are preserved under both under- and over-relaxation.

\paragraph{Discussion and future work.}
Our analysis focused on bounds for the expected regret in the noisy setting. An important next step would be to derive high-probability or "tail" bounds, potentially using concentration inequalities or arguments from martingale theory, which would provide stronger guarantees on the algorithm's performance on any given run. Finally, extending this powerful combination of regret analysis and operator theory to other iterative methods and non-linear models remains a challenging and exciting direction for future investigation.

\begingroup
\fontsize{9}{11}\selectfont 
\bibliographystyle{unsrt} 
\nocite{*}
\bibliography{reference_regret_mod}

@book{cima2006cauchy,
  title={The Cauchy transform},
  author={Cima, Joseph A and Matheson, Alec L and Ross, William T},
  number={125},
  year={2006},
  publisher={American Mathematical Soc.}
}

@book{katznelson2004introduction,
  title={An introduction to harmonic analysis},
  author={Katznelson, Yitzhak},
  year={2004},
  publisher={Cambridge University Press}
}

@article {MR4930266,
    AUTHOR = {Cesa-Bianchi, Nicol\`o and Eldowa, Khaled and Esposito, Emmanuel
              and Olkhovskaya, Julia},
     TITLE = {Improved regret bounds for bandits with expert advice},
   JOURNAL = {J. Artif. Intell. Res.},
  FJOURNAL = {Journal of Artificial Intelligence Research},
    VOLUME = {83},
      YEAR = {2025},
     PAGES = {Art. 6, 15},
      ISSN = {1076-9757},
   MRCLASS = {68T05},
  MRNUMBER = {4930266},
}

@article {MR4855849,
    AUTHOR = {Jaladi, Sri and Bistritz, Ilai},
     TITLE = {We are legion: high probability regret bound in adversarial
              multiagent online learning},
   JOURNAL = {IEEE Control Syst. Lett.},
  FJOURNAL = {IEEE Control Systems Letters},
    VOLUME = {8},
      YEAR = {2024},
     PAGES = {2985--2990},
   MRCLASS = {91A26},
  MRNUMBER = {4855849},
}

@article {MR4796482,
    AUTHOR = {Liang, Hao and Luo, Zhi-Quan},
     TITLE = {Bridging distributional and risk-sensitive reinforcement
              learning with provable regret bounds},
   JOURNAL = {J. Mach. Learn. Res.},
  FJOURNAL = {Journal of Machine Learning Research (JMLR)},
    VOLUME = {25},
      YEAR = {2024},
     PAGES = {Paper No. [221], 56},
      ISSN = {1532-4435},
   MRCLASS = {68T07},
  MRNUMBER = {4796482},
}

@article {MR4690295,
    AUTHOR = {Kirschner, Johannes and Lattimore, Tor and Krause, Andreas},
     TITLE = {Linear partial monitoring for sequential decision making
              algorithms, regret bounds and applications},
   JOURNAL = {J. Mach. Learn. Res.},
  FJOURNAL = {Journal of Machine Learning Research (JMLR)},
    VOLUME = {24},
      YEAR = {2023},
     PAGES = {Paper No. [346], 45},
      ISSN = {1532-4435},
   MRCLASS = {91B06 (68T05)},
  MRNUMBER = {4690295},
}

@article {MR4549442,
    AUTHOR = {Abbaszadeh Chekan, Jafar and Langbort, Cedric},
     TITLE = {Regret bounds for online-learning-based linear quadratic
              control under database attacks},
   JOURNAL = {Automatica J. IFAC},
  FJOURNAL = {Automatica. A Journal of IFAC, the International Federation of
              Automatic Control},
    VOLUME = {151},
      YEAR = {2023},
     PAGES = {Paper No. 110876, 10},
      ISSN = {0005-1098},
   MRCLASS = {93E35},
  MRNUMBER = {4549442},
       DOI = {10.1016/j.automatica.2023.110876},
       URL = {https://doi.org/10.1016/j.automatica.2023.110876},
}

@article {MR4396309,
    AUTHOR = {Wang, Zexin and Tan, Vincent Y. F. and Scarlett, Jonathan},
     TITLE = {Tight regret bounds for noisy optimization of a {B}rownian
              motion},
   JOURNAL = {IEEE Trans. Signal Process.},
  FJOURNAL = {IEEE Transactions on Signal Processing},
    VOLUME = {70},
      YEAR = {2022},
     PAGES = {1072--1087},
      ISSN = {1053-587X},
   MRCLASS = {94A12},
  MRNUMBER = {4396309},
       DOI = {10.1109/tsp.2022.3144939},
       URL = {https://doi.org/10.1109/tsp.2022.3144939},
}

@article {MR4066312,
    AUTHOR = {Ortner, Ronald},
     TITLE = {Regret bounds for reinforcement learning via {M}arkov chain
              concentration},
   JOURNAL = {J. Artif. Intell. Res.},
  FJOURNAL = {Journal of Artificial Intelligence Research},
    VOLUME = {67},
      YEAR = {2020},
     PAGES = {115--128},
      ISSN = {1076-9757},
   MRCLASS = {68T05 (60E15 60J20)},
  MRNUMBER = {4066312},
}

@misc{gunturk2019unrestricted,
  title={Unrestricted iterations of relaxed projections in Hilbert space: Regularity, absolute convergence, and statistics of displacements},
  author={G{\"u}nt{\"u}rk, C Sinan and Thao, Nguyen T},
  note={Technical Report, arXiv:1901.07516},
  year={2019}
}

@article{nikazad2024choosing,
  title={Choosing relaxation parameter in randomized Kaczmarz method},
  author={Nikazad, T and Khakzad, M},
  journal={J. Comput. Appl. Math.},
  volume={444},
  pages={115790},
  year={2024},
  publisher={Elsevier}
}

@inproceedings{zhangunconstrained,
  title={Unconstrained Robust Online Convex Optimization},
  author={Zhang, Jiujia and Cutkosky, Ashok},
  booktitle={Forty-second International Conference on Machine Learning}
}

@inproceedings{mcmahan2014unconstrained,
  title={Unconstrained online linear learning in hilbert spaces: Minimax algorithms and normal approximations},
  author={McMahan, H Brendan and Orabona, Francesco},
  booktitle={Conference on Learning Theory},
  pages={1020--1039},
  year={2014},
  organization={PMLR}
}

@article{strohmer2009randomized,
  title={A randomized Kaczmarz algorithm with exponential convergence},
  author={Strohmer, Thomas and Vershynin, Roman},
  journal={J. Fourier Anal. Appl.},
  volume={15},
  number={2},
  pages={262--278},
  year={2009},
  publisher={Springer}
}

@article{oswald2015convergence,
  title={Convergence analysis for Kaczmarz-type methods in a Hilbert space framework},
  author={Oswald, Peter and Zhou, Weiqi},
  journal={Linear Algebra Appl.},
  volume={478},
  pages={131--161},
  year={2015},
  publisher={Elsevier}
}

@article{needell2014paved,
  title={Paved with good intentions: analysis of a randomized block Kaczmarz method},
  author={Needell, Deanna and Tropp, Joel A},
  journal={Linear Algebra Appl.},
  volume={441},
  pages={199--221},
  year={2014},
  publisher={Elsevier}
}

@inproceedings{needell2014stochastic,
  title={Stochastic gradient descent, weighted sampling, and the randomized Kaczmarz algorithm},
  author={Needell, Deanna and Srebro, Nathan and Ward, Rachel},
  booktitle={Adv. Neural Inf. Process. Syst.},
  volume={27},
  year={2014}
}

@book {MR4932919,
    AUTHOR = {Berner, Chad},
     TITLE = {Frame-{L}ike {F}ourier {E}xpansions for {F}inite {B}orel
              {M}easures},
      NOTE = {Thesis (Ph.D.)--Iowa State University},
PUBLISHER = {ProQuest LLC, Ann Arbor, MI},
      YEAR = {2025},
     PAGES = {112},
      ISBN = {979-8286-43099-4},
   MRCLASS = {Thesis},
  MRNUMBER = {4932919},
       URL =
              {https://gateway.proquest.com/openurl?url_ver=Z39.88-2004&rft_val_fmt=info:ofi/fmt:kev:mtx:dissertation&res_dat=xri:pqm&rft_dat=xri:pqdiss:31936031},
}

@article {MR4506052,
    AUTHOR = {Alebrahim, R. and Thamburaja, P. and Srinivasa, A. and Reddy,
              J. N.},
     TITLE = {A robust {M}oore-{P}enrose pseudoinverse-based static
              finite-element solver for simulating non-local fracture in
              solids},
   JOURNAL = {Comput. Methods Appl. Mech. Engrg.},
  FJOURNAL = {Computer Methods in Applied Mechanics and Engineering},
    VOLUME = {403},
      YEAR = {2023},
    NUMBER = {part A},
     PAGES = {Paper No. 115727, 26},
      ISSN = {0045-7825},
   MRCLASS = {74S05 (74A20 74R10)},
  MRNUMBER = {4506052},
       DOI = {10.1016/j.cma.2022.115727},
       URL = {https://doi.org/10.1016/j.cma.2022.115727},
}

@article {MR3921904,
    AUTHOR = {Klimczak, Marek and Cecot, Witold},
     TITLE = {On {M}oore-{P}enrose pseudoinverse computation for stiffness
              matrices resulting from higher order approximation},
   JOURNAL = {Math. Probl. Eng.},
  FJOURNAL = {Mathematical Problems in Engineering},
      YEAR = {2019},
     PAGES = {Art. ID 5060397, 16},
      ISSN = {1024-123X},
   MRCLASS = {65F05 (15A09)},
  MRNUMBER = {3921904},
       DOI = {10.1155/2019/5060397},
       URL = {https://doi.org/10.1155/2019/5060397},
}

@article {MR1892841,
    AUTHOR = {Boman, Eugene C.},
     TITLE = {The {M}oore-{P}enrose pseudoinverse of an arbitrary, square,
              {$k$}-circulant matrix},
   JOURNAL = {Linear Multilinear Algebra},
  FJOURNAL = {Linear and Multilinear Algebra},
    VOLUME = {50},
      YEAR = {2002},
    NUMBER = {2},
     PAGES = {175--179},
      ISSN = {0308-1087},
   MRCLASS = {15A09},
  MRNUMBER = {1892841},
       DOI = {10.1080/03081080290019559},
       URL = {https://doi.org/10.1080/03081080290019559},
}

@article {MR3711149,
    AUTHOR = {Luger, Annemarie and Nedic, Mitja},
     TITLE = {A characterization of {H}erglotz-{N}evanlinna functions in two
              variables via integral representations},
   JOURNAL = {Ark. Mat.},
  FJOURNAL = {Arkiv f\"{o}r Matematik},
    VOLUME = {55},
      YEAR = {2017},
    NUMBER = {1},
     PAGES = {199--216},
      ISSN = {0004-2080},
   MRCLASS = {32A26 (30E20 32A10 32A30 32A40 32A70)},
  MRNUMBER = {3711149},
MRREVIEWER = {George Chailos},
       DOI = {10.4310/ARKIV.2017.v55.n1.a10},
       URL = {https://doi.org/10.4310/ARKIV.2017.v55.n1.a10},
}

@article{shalev2012online,
  title={Online learning and online convex optimization},
  author={Shalev-Shwartz, Shai and others},
  journal={Found. Trends Mach. Learn.},
  volume={4},
  number={2},
  pages={107--194},
  year={2012},
  publisher={Now Publishers, Inc.}
}

@article{hazan2016introduction,
  title={Introduction to online convex optimization},
  author={Hazan, Elad and others},
  journal={Found. Trends Optim.},
  volume={2},
  number={3-4},
  pages={157--325},
  year={2016},
  publisher={Now Publishers, Inc.}
}

@article{russo2018tutorial,
  title={A tutorial on thompson sampling},
  author={Russo, Daniel J and Van Roy, Benjamin and Kazerouni, Abbas and Osband, Ian and Wen, Zheng and others},
  journal={Found. Trends Mach. Learn.},
  volume={11},
  number={1},
  pages={1--96},
  year={2018},
  publisher={Now Publishers, Inc.}
}

@book{rudin1987real,
  title={Real and complex analysis},
  author={Rudin, Walter},
  year={1987},
  publisher={McGraw-Hill, Inc.}
}

@article{kwapien2001kaczmarz,
  title={On the Kaczmarz algorithm of approximation in infinite-dimensional spaces},
  author={Kwapie{\'n}, Stanis{\l}aw and Mycielski, Jan},
  journal={Studia Math.},
  volume={148},
  number={1},
  pages={75--86},
  year={2001},
  publisher={Institute of Mathematics Polish Academy of Sciences}
}

@article{kwapien2006erratum,
  title={Erratum to the paper" On the Kaczmarz algorithm of approximation in infinite-dimensional spaces"(Studia Math. 148 (2001), 75-86)},
  author={Kwapie{\'n}, Stanis{\l}aw and Mycielski, Jan},
  journal={Studia Math.},
  volume={176},
  number={1},
  year={2006},
  publisher={Institute of Mathematics Polish Academy of Sciences}
}

@article{jorgensen2020kaczmarz,
  title={A Kaczmarz algorithm for sequences of projections, infinite products, and applications to frames in IFS L 2 spaces},
  author={Jorgensen, Palle and Song, Myung-Sin and Tian, James},
  journal={Adv. Oper. Theory},
  volume={5},
  number={3},
  pages={1100--1131},
  year={2020},
  publisher={Springer}
}

@inproceedings{evron2022catastrophic,
  title={How catastrophic can catastrophic forgetting be in linear regression?},
  author={Evron, Itay and Moroshko, Edward and Ward, Rachel and Srebro, Nathan and Soudry, Daniel},
  booktitle={Conference on Learning Theory},
  pages={4028--4079},
  year={2022},
  organization={PMLR}
}

@misc{berner2024operator,
  title={Operator orbit frames and frame-like Fourier expansions},
  author={Berner, Chad and Weber, Eric S},
  note={Technical Report, arXiv:2409.10706},
  year={2024}
}

@incollection{hegde2021kaczmarz,
  title={A Kaczmarz algorithm for solving tree based distributed systems of equations},
  author={Hegde, Chinmay and Keinert, Fritz and Weber, Eric S},
  booktitle={Excursions in Harmonic Analysis, Volume 6: In Honor of John Benedetto’s 80th Birthday},
  pages={385--411},
  year={2021},
  publisher={Springer}
}

@incollection{herr2020harmonic,
  title={Harmonic Analysis of Fractal Measures: Basis and Frame Algorithms for Fractal L 2-Spaces, and Boundary Representations as Closed Subspaces of the Hardy Space},
  author={Herr, John E and Jorgensen, Palle ET and Weber, Eric S},
  booktitle={Analysis, probability and mathematical physics on fractals},
  pages={163--221},
  year={2020},
  publisher={World Scientific}
}

@article{herr2017fourier,
  title={Fourier series for singular measures},
  author={Herr, John E and Weber, Eric S},
  journal={Axioms},
  volume={6},
  number={2},
  pages={7},
  year={2017},
  publisher={MDPI}
}

@incollection{herr2023fourier,
  title={Fourier series for fractals in two dimensions},
  author={Herr, John E and Jorgensen, Palle ET and Weber, Eric S},
  booktitle={From Classical Analysis to Analysis on Fractals: A Tribute to Robert Strichartz, Volume 1},
  pages={183--229},
  year={2023},
  publisher={Springer}
}

@article{berner2025fourier,
  title={Fourier series for singular measures in higher dimensions},
  author={Berner, Chad and Herr, John E and Jorgensen, Palle ET and Weber, Eric S},
  journal={J. Fourier Anal. Appl.},
  volume={31},
  number={1},
  pages={1--33},
  year={2025},
  publisher={Springer}
}

@article{herr2019positive,
  title={Positive matrices in the Hardy space with prescribed boundary representations via the Kaczmarz algorithm},
  author={Herr, John E. and Jorgensen, Palle E. T. and Weber, Eric S.},
  journal={J. Anal. Math.},
  volume={138},
  number={1},
  pages={209--234},
  year={2019},
  publisher={Springer},
  doi={10.1007/s11854-019-0026-6},
  url={https://doi.org/10.1007/s11854-019-0026-6}
}

@book{jorgensen2023mathematics,
  title={Mathematics of Multilevel Systems—Data, Scaling, Images, Signals, and Fractals},
  author={Jorgensen, Palle E. T. and Song, Myung-Sin},
  series={Contemporary Mathematics and Its Applications: Monographs, Expositions and Lecture Notes},
  volume={8},
  publisher={World Scientific Publishing Co. Pte. Ltd., Hackensack, NJ},
  year={2023},
  pages={xv+253},
  isbn={9789811268977, 9789811268991, 9789811269011},
  doi={10.1142/13227},
  url={https://doi.org/10.1142/13227}
}

@article {MR4835057,
    AUTHOR = {Lu, Kaihong},
     TITLE = {Online distributed algorithms for online noncooperative games
              with stochastic cost functions: high probability bound of
              regrets},
   JOURNAL = {IEEE Trans. Automat. Control},
  FJOURNAL = {Institute of Electrical and Electronics Engineers.
              Transactions on Automatic Control},
    VOLUME = {69},
      YEAR = {2024},
    NUMBER = {12},
     PAGES = {8860--8867},
      ISSN = {0018-9286},
   MRCLASS = {91A10 (68M14 91A68 93E20)},
  MRNUMBER = {4835057},
       DOI = {10.1109/tac.2024.3419018},
       URL = {https://doi.org/10.1109/tac.2024.3419018},
}

@article {MR3049487,
    AUTHOR = {Gerchinovitz, S\'{e}bastien},
     TITLE = {Sparsity regret bounds for individual sequences in online
              linear regression},
   JOURNAL = {J. Mach. Learn. Res.},
  FJOURNAL = {Journal of Machine Learning Research (JMLR)},
    VOLUME = {14},
      YEAR = {2013},
     PAGES = {729--769},
      ISSN = {1532-4435},
   MRCLASS = {62J02 (62G05 62J07 62L99)},
  MRNUMBER = {3049487},
}

@article {MR4689556,
    AUTHOR = {Alpay, Daniel and Bhattacharyya, Tirthankar and Jindal, Abhay
              and Kumar, Poornendu},
     TITLE = {A dilation theoretic approach to approximation by inner
              functions},
   JOURNAL = {Bull. Lond. Math. Soc.},
  FJOURNAL = {Bulletin of the London Mathematical Society},
    VOLUME = {55},
      YEAR = {2023},
    NUMBER = {6},
     PAGES = {2840--2855},
      ISSN = {0024-6093},
   MRCLASS = {47A20 (30H50 32A37 47A56 93B28)},
  MRNUMBER = {4689556},
MRREVIEWER = {Ramlal Debnath},
       DOI = {10.1112/blms.12897},
       URL = {https://doi.org/10.1112/blms.12897},
}

@incollection {MR1320542,
    AUTHOR = {Alpay, Daniel and Baratchart, Laurent and Gombani, Andrea},
     TITLE = {On the differential structure of matrix-valued rational inner
              functions},
BOOKTITLE = {Nonselfadjoint operators and related topics ({B}eer {S}heva,
              1992)},
    SERIES = {Oper. Theory Adv. Appl.},
    VOLUME = {73},
     PAGES = {30--66},
PUBLISHER = {Birkh\"{a}user, Basel},
      YEAR = {1994},
   MRCLASS = {30D50 (46E22 47A57 47N70 93B29)},
  MRNUMBER = {1320542},
MRREVIEWER = {Joseph A. Ball},
}

@article {MR770982,
    AUTHOR = {Rudin, Walter},
     TITLE = {Composition with inner functions},
   JOURNAL = {Complex Variables Theory Appl.},
  FJOURNAL = {Complex Variables. Theory and Application. An International
              Journal},
    VOLUME = {4},
      YEAR = {1984},
    NUMBER = {1},
     PAGES = {7--19},
      ISSN = {0278-1077},
   MRCLASS = {32A40 (31C10)},
  MRNUMBER = {770982},
MRREVIEWER = {S. H. Kon},
       DOI = {10.1080/17476938408814087},
       URL = {https://doi.org/10.1080/17476938408814087},
}

@book {MR255841,
    AUTHOR = {Rudin, Walter},
     TITLE = {Function theory in polydiscs},
PUBLISHER = {W. A. Benjamin, Inc., New York-Amsterdam},
      YEAR = {1969},
     PAGES = {vii+188},
   MRCLASS = {32.00},
  MRNUMBER = {255841},
MRREVIEWER = {W. Stoll},
}

@book {MR3495345,
    AUTHOR = {Christensen, Ole},
     TITLE = {An introduction to frames and {R}iesz bases},
    SERIES = {Applied and Numerical Harmonic Analysis},
   EDITION = {Second},
PUBLISHER = {Birkh\"{a}user/Springer, [Cham]},
      YEAR = {2016},
     PAGES = {xxv+704},
      ISBN = {978-3-319-25611-5; 978-3-319-25613-9},
   MRCLASS = {42-02 (42C15 42C40 46B15 46C05)},
  MRNUMBER = {3495345},
MRREVIEWER = {Marcin M. Bownik},
       DOI = {10.1007/978-3-319-25613-9},
       URL = {https://doi.org/10.1007/978-3-319-25613-9},
}

@article {MR3061703,
    AUTHOR = {Christensen, Ole and Osgooei, Elnaz},
     TITLE = {On frame properties for {F}ourier-like systems},
   JOURNAL = {J. Approx. Theory},
  FJOURNAL = {Journal of Approximation Theory},
    VOLUME = {172},
      YEAR = {2013},
     PAGES = {47--57},
      ISSN = {0021-9045},
   MRCLASS = {42C15 (41A30 42A10 42C40)},
  MRNUMBER = {3061703},
MRREVIEWER = {Patricia Mariela Morillas},
       DOI = {10.1016/j.jat.2013.04.004},
       URL = {https://doi.org/10.1016/j.jat.2013.04.004},
}

@incollection {MR4696783,
    AUTHOR = {Christensen, Ole and Hasannasab, Marzieh},
     TITLE = {A survey on frame representations and operator orbits},
BOOKTITLE = {Sampling, approximation, and signal analysis---harmonic
              analysis in the spirit of {J}. {R}owland {H}iggins},
    SERIES = {Appl. Numer. Harmon. Anal.},
     PAGES = {349--370},
PUBLISHER = {Birkh\"{a}user/Springer, Cham},
      YEAR = {2023},
     NOTE = {\copyright 2023},
   MRCLASS = {42C15},
  MRNUMBER = {4696783},
MRREVIEWER = {Morteza Mirzaee Azandaryani},
       DOI = {10.1007/978-3-031-41130-4\_13},
       URL = {https://doi.org/10.1007/978-3-031-41130-4_13},
}

@incollection {MR3793303,
    AUTHOR = {Christensen, Ole and Hasannasab, Marzieh},
     TITLE = {Frames, operator representations, and open problems},
BOOKTITLE = {The diversity and beauty of applied operator theory},
    SERIES = {Oper. Theory Adv. Appl.},
    VOLUME = {268},
     PAGES = {155--165},
PUBLISHER = {Birkh\"{a}user/Springer, Cham},
      YEAR = {2018},
   MRCLASS = {42C15},
  MRNUMBER = {3793303},
MRREVIEWER = {Yuri A. Farkov},
}

@incollection {MR4633118,
    AUTHOR = {Benedetto, John J. and Koprowski, Paul J. and Nolan, John S.},
     TITLE = {A generalization of {G}leason's frame function for quantum
              measurement},
BOOKTITLE = {Theoretical physics, wavelets, analysis, genomics---an
              indisciplinary tribute to {A}lex {G}rossmann},
    SERIES = {Appl. Numer. Harmon. Anal.},
     PAGES = {463--513},
PUBLISHER = {Birkh\"{a}user/Springer, Cham},
      YEAR = {2023},
     NOTE = {\copyright 2023},
   MRCLASS = {81P15 (42C15)},
  MRNUMBER = {4633118},
       DOI = {10.1007/978-3-030-45847-8\_21},
       URL = {https://doi.org/10.1007/978-3-030-45847-8_21},
}

@article {MR2735759,
    AUTHOR = {Dutkay, Dorin Ervin and Han, Deguang and Sun, Qiyu and Weber,
              Eric},
     TITLE = {On the {B}eurling dimension of exponential frames},
   JOURNAL = {Adv. Math.},
  FJOURNAL = {Advances in Mathematics},
    VOLUME = {226},
      YEAR = {2011},
    NUMBER = {1},
     PAGES = {285--297},
      ISSN = {0001-8708},
   MRCLASS = {42C15 (28A78 28A80)},
  MRNUMBER = {2735759},
MRREVIEWER = {Alfredo Lazaro Gonz\'{a}lez},
       DOI = {10.1016/j.aim.2010.06.017},
       URL = {https://doi.org/10.1016/j.aim.2010.06.017},
}

@article {MR2085421,
    AUTHOR = {Dykema, Ken and Freeman, Dan and Kornelson, Keri and Larson,
              David and Ordower, Marc and Weber, Eric},
     TITLE = {Ellipsoidal tight frames and projection decompositions of
              operators},
   JOURNAL = {Illinois J. Math.},
  FJOURNAL = {Illinois Journal of Mathematics},
    VOLUME = {48},
      YEAR = {2004},
    NUMBER = {2},
     PAGES = {477--489},
      ISSN = {0019-2082},
   MRCLASS = {42C15 (42C40 46C05 47B99)},
  MRNUMBER = {2085421},
       URL = {http://projecteuclid.org/euclid.ijm/1258138393},
}

@incollection {MR4696781,
    AUTHOR = {K\"{o}hldorfer, L. and Balazs, P. and Casazza, P. and Heineken, S.
              and Hollomey, C. and Morillas, P. and Shamsabadi, M.},
     TITLE = {A survey of fusion frames in {H}ilbert spaces},
BOOKTITLE = {Sampling, approximation, and signal analysis---harmonic
              analysis in the spirit of {J}. {R}owland {H}iggins},
    SERIES = {Appl. Numer. Harmon. Anal.},
     PAGES = {245--328},
PUBLISHER = {Birkh\"{a}user/Springer, Cham},
      YEAR = {2023},
     NOTE = {\copyright 2023},
   MRCLASS = {42C15},
  MRNUMBER = {4696781},
MRREVIEWER = {Animesh Bhandari},
       DOI = {10.1007/978-3-031-41130-4\_11},
       URL = {https://doi.org/10.1007/978-3-031-41130-4_11},
}

@article {MR3488069,
    AUTHOR = {Casazza, Peter G. and Pinkham, Eric and Tuomanen, Brian},
     TITLE = {Riesz outer product {H}ilbert space frames: quantitative
              bounds, topological properties, and full geometric
              characterization},
   JOURNAL = {J. Math. Anal. Appl.},
  FJOURNAL = {Journal of Mathematical Analysis and Applications},
    VOLUME = {441},
      YEAR = {2016},
    NUMBER = {1},
     PAGES = {475--498},
      ISSN = {0022-247X},
   MRCLASS = {42C15},
  MRNUMBER = {3488069},
MRREVIEWER = {Pierluigi Vellucci},
       DOI = {10.1016/j.jmaa.2016.04.001},
       URL = {https://doi.org/10.1016/j.jmaa.2016.04.001},
}

@article {MR4906562,
    AUTHOR = {Jirakitpuwapat, Wachirapong},
     TITLE = {A regret bound for the {A}da{M}ax algorithm with image
              segmentation application},
   JOURNAL = {Math. Methods Appl. Sci.},
  FJOURNAL = {Mathematical Methods in the Applied Sciences},
    VOLUME = {48},
      YEAR = {2025},
    NUMBER = {9},
     PAGES = {10208--10214},
      ISSN = {0170-4214},
   MRCLASS = {90C15 (65K10 68U10)},
  MRNUMBER = {4906562},
       DOI = {10.1002/mma.10879},
       URL = {https://doi.org/10.1002/mma.10879},
}

@article {MR4883940,
    AUTHOR = {Mussi, Marco and Metelli, Alberto Maria},
     TITLE = {Generalizing the regret: an analysis of lower and upper
              bounds},
   JOURNAL = {J. Artif. Intell. Res.},
  FJOURNAL = {Journal of Artificial Intelligence Research},
    VOLUME = {82},
      YEAR = {2025},
     PAGES = {1773--1806},
      ISSN = {1076-9757},
   MRCLASS = {68T37 (68W27)},
  MRNUMBER = {4883940},
}

@article {MR4635158,
    AUTHOR = {Wu, Changlong and Heidari, Mohsen and Grama, Ananth and
              Szpankowski, Wojciech},
     TITLE = {Regret bounds for log-loss via {B}ayesian algorithms},
   JOURNAL = {IEEE Trans. Inf. Theory},
  FJOURNAL = {Institute of Electrical and Electronics Engineers.
              Transactions on Information Theory},
    VOLUME = {69},
      YEAR = {2023},
    NUMBER = {9},
     PAGES = {5971--5989},
      ISSN = {0018-9448},
   MRCLASS = {62-08 (62C10 62L05)},
  MRNUMBER = {4635158},
       DOI = {10.1109/tit.2023.3279197},
       URL = {https://doi.org/10.1109/tit.2023.3279197},
}

@article {MR4531092,
    AUTHOR = {Li, Xiaocheng and Ye, Yinyu},
     TITLE = {Online linear programming: dual convergence, new algorithms,
              and regret bounds},
   JOURNAL = {Oper. Res.},
  FJOURNAL = {Operations Research},
    VOLUME = {70},
      YEAR = {2022},
    NUMBER = {5},
     PAGES = {2948--2966},
      ISSN = {0030-364X},
   MRCLASS = {90C15 (68W27 90C05 90C39)},
  MRNUMBER = {4531092},
}

@article {MR4211383,
    AUTHOR = {Ziemann, Ingvar and Sandberg, Henrik},
     TITLE = {Regret lower bounds for unbiased adaptive control of linear
              quadratic regulators},
   JOURNAL = {IEEE Control Syst. Lett.},
  FJOURNAL = {IEEE Control Systems Letters},
    VOLUME = {4},
      YEAR = {2020},
    NUMBER = {3},
     PAGES = {785--790},
   MRCLASS = {93E35 (62H12)},
  MRNUMBER = {4211383},
MRREVIEWER = {\L . Stettner},
}
\endgroup

\appendix

\section{The Proof of Lemma~\ref{lem:equivalence_alpha}}
We define two sequences $\alpha_n^{(\lambda)}$ and $ \beta_n^{(\lambda)}$ as follows:

\begin{lemma}
\begin{itemize}
    \item \textbf{Combinatorial:}
    
    Let $P_n$ be the set of compositions of $n$, i.e., $P_n = \{(n_1, n_2, \dots, n_k) \;|\; k, n_i \in \mathbb{N}^+,  n_1 + n_2 + \dots + n_k = n\}$ and $l(p)$ be the length of the tuple $p \in P_n$. Let
    
    \begin{align*} \alpha_0^{(\lambda)} &= \lambda, \\ \alpha_n^{(\lambda)} &= \sum_{p\in P_n}(-1)^{\ell(p)}\,\lambda^{\ell(p)+1}\prod_{j=1}^{\ell(p)}\widehat{\mu}(p_j), \quad (n \ge 1). \end{align*}
    \item \textbf{Recursive:} Let
    \begin{align*} \beta_0^{(\lambda)} &= \lambda, \\ \beta_n^{(\lambda)} &= -\lambda \sum_{k=0}^{n-1} \widehat{\mu}(n-k) \beta_k^{(\lambda)}, \quad (n \ge 1). \end{align*}
\end{itemize}
\label{lem:equivalence_alpha_appendix}
Then, we have \(\alpha_n^{(\lambda)} = \beta_n^{(\lambda)}\) for all \(n \ge 0\).
\end{lemma}

\begin{proof} [Proof of lemma]

We show that the combinatorial definition satisfies the recursion.
The base case \(\alpha_0^{(\lambda)} = \lambda\) matches \(\beta_0^{(\lambda)}\).
Assume the combinatorial formula holds for \(\alpha_k^{(\lambda)}\) for \(0 \le k \le n\). We want to show that \(-\lambda \sum_{k=0}^{n} \widehat{\mu}(n+1-k) \alpha_k^{(\lambda)}\) equals the combinatorial formula for \(\alpha_{n+1}^{(\lambda)}\).
\begin{align*} -\lambda \sum_{k=0}^{n} \widehat{\mu}(n+1-k) \alpha_k^{(\lambda)} &= -\lambda \widehat{\mu}(n+1) \alpha_0^{(\lambda)} - \lambda \sum_{k=1}^{n} \widehat{\mu}(n+1-k) \alpha_k^{(\lambda)} \\ &= -\lambda^2 \widehat{\mu}(n+1) - \lambda \sum_{k=1}^{n} \widehat{\mu}(n+1-k) \left( \sum_{p\in P_k}(-1)^{\ell(p)}\,\lambda^{\ell(p)+1}\prod_{w=1}^{\ell(p)}\widehat{\mu}(p_w) \right) \\ &= -\lambda^2 \widehat{\mu}(n+1) + \sum_{k=1}^{n} \widehat{\mu}(n+1-k) \left( \sum_{p\in P_k}(-1)^{\ell(p)+1}\,\lambda^{\ell(p)+2}\prod_{w=1}^{\ell(p)}\widehat{\mu}(p_w) \right). \end{align*}
Now, consider the combinatorial formula for \(\alpha_{n+1}^{(\lambda)}\). Split the sum based on the first part \( m = n+1-k\), where \(p'=(m, p)\) and \(p \in P_k\):
\begin{align*} \alpha_{n+1}^{(\lambda)} &= \sum_{p'\in P_{n+1}}(-1)^{\ell(p')}\,\lambda^{\ell(p')+1}\prod_{j=1}^{\ell(p')}\widehat{\mu}(p'_j) \\ &= \sum_{m=1}^{n+1} \sum_{p\in P_{n+1-m}} (-1)^{\ell(p)+1}\,\lambda^{(\ell(p)+1)+1} \widehat{\mu}(m) \prod_{w=1}^{\ell(p)}\widehat{\mu}(p_w) \\ &= \sum_{k=0}^{n} \sum_{p\in P_{k}} (-1)^{\ell(p)+1}\,\lambda^{\ell(p)+2} \widehat{\mu}(n+1-k) \prod_{w=1}^{\ell(p)}\widehat{\mu}(p_w) \quad (\text{let } k=n+1-m) \\ &= \underbrace{(-1)^{0+1}\,\lambda^{0+2} \widehat{\mu}(n+1) \prod \emptyset}_{k=0 \text{ term}} + \sum_{k=1}^{n} \widehat{\mu}(n+1-k) \left( \sum_{p\in P_{k}} (-1)^{\ell(p)+1}\,\lambda^{\ell(p)+2} \prod_{w=1}^{\ell(p)}\widehat{\mu}(p_w) \right) \\ &= -\lambda^2 \widehat{\mu}(n+1) + \sum_{k=1}^{n} \widehat{\mu}(n+1-k) \left( \sum_{p\in P_{k}} (-1)^{\ell(p)+1}\,\lambda^{\ell(p)+2} \prod_{w=1}^{\ell(p)}\widehat{\mu}(p_w) \right). \end{align*}
Comparing the expression derived from the recursion and the one derived from the combinatorial sum, we see they are identical. Thus, the combinatorial definition implies the recursion. Since they share the same base case \(\alpha_0^{(\lambda)} = \beta_0^{(\lambda)} = \lambda\), we conclude \(\alpha_n^{(\lambda)} = \beta_n^{(\lambda)}\) for all \(n \ge 0\).
This justifies using the recursive property within the inductive proof of the lemma relating \(h_n^{(\lambda)}\) and \(\alpha_n^{(\lambda)}\).
\end{proof}

\section{Derivation of \eqref{eq:shift-reduction-fixed} and \eqref{eq:i0-identity-fixed}.}

We start with summarizing the notations and some known facts below. 
The inner product is conjugate–linear in the first argument and linear in the second, and the system is stationary, so that 
\[
\langle e_s,e_i\rangle=\widehat\mu(s-i),\qquad
\widehat\mu(-t)=\overline{\widehat\mu(t)},\qquad \widehat\mu(0)=1.
\]
Fix \(\lambda\in(0,2)\) and define the coefficients $\alpha^{(\lambda)}_n$ by
\begin{align}
\label{R1}
\alpha^{(\lambda)}_0=\lambda,\qquad
\alpha^{(\lambda)}_n=-\lambda\sum_{t=1}^n\widehat\mu(t)\,\alpha^{(\lambda)}_{n-t}\quad(n\ge1).
\end{align}

Since \(\lambda\in\mathbb R\), taking conjugates gives
\begin{align}
\label{R2}
\overline{\alpha^{(\lambda)}_n}
=-\lambda\sum_{t=1}^n\widehat\mu(-t)\,\overline{\alpha^{(\lambda)}_{n-t}}
\quad(n\ge1).
\end{align}

Let
\[
h_k^{(\lambda)}:=\sum_{r=0}^k \overline{\alpha^{(\lambda)}_{k-r}}\,e_r,
\qquad
u_n:=\sum_{l=0}^n \alpha^{(\lambda)}_l\,e_l .
\]

By the conjugate linearity, for all \(k,i\ge0\),
\[
\langle h_k^{(\lambda)},e_i\rangle
=\sum_{s=0}^k \alpha^{(\lambda)}_{k-s}\,\widehat\mu(s-i),
\]
and hence, 
\begin{align}
\label{ME} 
K_n(e_i)=\sum_{k=0}^{n} e_k \sum_{s=0}^{k} \alpha^{(\lambda)}_{k-s}\,\widehat\mu(s-i).
\end{align}

\subsection*{1.\;The case \(i=0\)}
If we set \(i=0\) in \eqref{ME}, then 
\[
K_n(e_0)=\sum_{k=0}^{n} e_k\Bigl(\alpha^{(\lambda)}_{k}\widehat\mu(0)
+\sum_{s=1}^{k}\alpha^{(\lambda)}_{k-s}\widehat\mu(s)\Bigr)
=u_n+\sum_{k=1}^{n}e_k\sum_{s=1}^{k}\alpha^{(\lambda)}_{k-s}\widehat\mu(s).
\]
By \eqref{R1}, for \(k\ge1\),
\(\sum_{s=1}^{k}\widehat\mu(s)\,\alpha^{(\lambda)}_{k-s}
=-\lambda^{-1}\alpha^{(\lambda)}_{k}\).
Therefore,
\[
K_n(e_0)=u_n-\frac1\lambda\sum_{k=1}^{n}\alpha^{(\lambda)}_{k}\,e_k
=u_n-\frac1\lambda u_n+e_0,
\]
and hence,
\[
\;
K_n(e_0)-e_0=\frac{\lambda-1}{\lambda}\,u_n,\qquad n\ge0.
\;
\tag{I0}
\]

\subsection*{2.\;The case for \(i\ge1\)}
Assume \(n\ge i\) and split \(K_n(e_i)=a_n+b_n\) as
\[
a_n:=\sum_{k=0}^{n} e_k \sum_{s=0}^{\min\{k,i\}} \alpha^{(\lambda)}_{k-s}\,\widehat\mu(s-i),
\qquad
b_n:=\sum_{k=0}^{n} e_k \sum_{s=i+1}^{k} \alpha^{(\lambda)}_{k-s}\,\widehat\mu(s-i).
\]

\begin{itemize}
    \item Block \(a_n\) (\(s\le i\)):
Let \(j=i-s\in\{0,\dots,i\}\) and \(l=k-s\). Then,
\[
a_n=\sum_{j=0}^{i}\widehat\mu(-j)\sum_{l=0}^{\,n+j-i}\alpha^{(\lambda)}_{l}\,e_{\,l+i-j}.
\]
We define $v_{n,i}$ and rewrite $a_n$ as follows:
\[
v_{n,i}:=\sum_{l=0}^{\,n-i}\alpha^{(\lambda)}_{l}\,e_{\,l+i},
\qquad
a_n=v_{n,i}+\sum_{j=1}^{i}\widehat\mu(-j)\sum_{l=0}^{\,n+j-i}\alpha^{(\lambda)}_{l}\,e_{\,l+i-j}.
\]

    \item Block \(b_n\) (\(s>i\)):
Let \(t:=s-i\ge1\) and \(m:=k-s\ge0\). Then,
\[
b_n=\sum_{t=1}^{\infty}\widehat\mu(t)\sum_{m=0}^{\,n-i-t}\alpha^{(\lambda)}_{m}\,e_{\,m+i+t},
\]
with the inner sum \(=0\) if \(n-i-t<0\).
\end{itemize}

Now, set \(M:=n-i\). We claim

\begin{align}
\label{Tail}
\; e_i \;=\; b_n \;+\; \frac{1}{\lambda}\,v_{n,i}. \;
\end{align}

Indeed, for any \(r\in\mathbb Z\),
\begin{align*}
\langle b_n,e_{i+r}\rangle
&=\sum_{t=1}^{\infty}\overline{\widehat\mu(t)}\sum_{m=0}^{M-t}
\overline{\alpha^{(\lambda)}_{m}}\;\langle e_{m+i+t},e_{i+r}\rangle \\
&=\sum_{t=1}^{\infty}\widehat\mu(-t)\sum_{m=0}^{M-t}
\overline{\alpha^{(\lambda)}_{m}}\,\widehat\mu(m+t-r) \\
&=\sum_{l=1}^{M}\Bigl(\sum_{t=1}^{l}\widehat\mu(-t)\,\overline{\alpha^{(\lambda)}_{l-t}}\Bigr)\widehat\mu(l-r)
=-\frac{1}{\lambda}\sum_{l=1}^{M}\overline{\alpha^{(\lambda)}_{l}}\,\widehat\mu(l-r),
\end{align*}
where the last step uses \eqref{R2}. Also,
\[
\frac{1}{\lambda}\,\langle v_{n,i},e_{i+r}\rangle
=\frac{1}{\lambda}\sum_{l=0}^{M}\overline{\alpha^{(\lambda)}_{l}}\,\widehat\mu(l-r).
\]
Then,
\(\langle b_n+\tfrac{1}{\lambda}v_{n,i},e_{i+r}\rangle
=\tfrac{1}{\lambda}\overline{\alpha^{(\lambda)}_{0}}\widehat\mu(-r)
=\widehat\mu(-r)=\langle e_i,e_{i+r}\rangle\),
which proves \eqref{Tail}.

Finally, combining the results we have so far, we have
\[
K_n(e_i)-e_i=a_n+b_n-e_i
=\Bigl(v_{n,i}+\sum_{j=1}^{i}\widehat\mu(-j)\sum_{l=0}^{\,n+j-i}\alpha^{(\lambda)}_{l}\,e_{\,l+i-j}\Bigr)
-\frac{1}{\lambda}v_{n,i},
\]
hence
\begin{align}
\label{SR lambda}
\;
K_n(e_i)-e_i
=\frac{\lambda-1}{\lambda}\,v_{n,i}
+\sum_{j=1}^{i}\widehat\mu(-j)\sum_{l=0}^{\,n+j-i}\alpha^{(\lambda)}_{l}\,e_{\,l+i-j},
\qquad i\ge1,\; n\ge i.
\;
\end{align}

\end{document}